\documentclass{article}

\PassOptionsToPackage{numbers, compress}{natbib}

\usepackage[final]{neurips_2022}

\usepackage[utf8]{inputenc} %
\usepackage[T1]{fontenc}    %
\usepackage{hyperref}       %
\usepackage{url}            %
\usepackage{booktabs}       %
\usepackage{amsfonts}       %
\usepackage{nicefrac}       %
\usepackage{microtype}      %
\usepackage{xcolor}         %
\usepackage{array}
\usepackage{multirow}
\usepackage{tikz}
\usepackage{ctable}

\title{First Hitting Diffusion Models for Generating Manifold, Graph and Categorical Data}

\author{%
  Mao Ye,\thanks{Corresponding author. Email: maoye21@utexas.edu}\ \ \ \ \ \ Lemeng Wu,\ \ \ \ \ \ Qiang Liu \\
  Department of Computer Science\\
  The University of Texas at Austin\\
}

\newcommand{\ours}{FHDM }
\usepackage{qiangstyle} 

\begin{document}

\global\long\def\PP{\text{Pr}}%
\global\long\def\S{\mathcal{S}}%
\global\long\def\R{\mathbb{R}}%
\global\long\def\T{\mathcal{T}}%
\global\long\def\W{\mathcal{W}}%
\global\long\def\t{t}%

\global\long\def\bx{\bar{X}}%
\global\long\def\bt{\bar{\tau}}%
\global\long\def\by{\bar{Y}}%

\newtheorem{theorem}{Theorem}
\newtheorem{lemma}{Lemma}
\newtheorem{proposition}{Proposition}
\newtheorem{assumption}{Assumption}
\newtheorem{definition}{Definition}

\maketitle
\begin{abstract}
We propose a family of First Hitting Diffusion Models (FHDM), deep generative models that generate data with a diffusion process that terminates at a random first hitting time. This yields an extension of the standard fixed-time diffusion models that terminate at a pre-specified deterministic time. Although standard diffusion models are designed for continuous unconstrained data, FHDM is naturally designed to learn distributions on continuous as well as a range of discrete and structure domains. Moreover, FHDM  enables instance-dependent terminate time and accelerates the diffusion process to sample higher quality data with fewer diffusion steps. Technically, we train FHDM by maximum likelihood estimation on diffusion trajectories augmented from observed data with conditional first hitting processes (i.e., bridge) derived based on Doob's $h$-transform, deviating from the commonly used time-reversal mechanism. We apply FHDM to generate data in various domains such as point cloud (general continuous distribution),  climate and geographical events on earth (continuous distribution on the sphere),  unweighted graphs (distribution of binary matrices), and segmentation maps of 2D images (high-dimensional categorical distribution). We observe considerable improvement compared with the state-of-the-art approaches in both quality and speed.

\end{abstract}

\section{Introduction}
Diffusion processes have become a powerful tool in various areas of machine learning (ML) and statistics.
Traditionally, Langevin dynamics and Hamiltonian Monte Carlo have been  foundations for learning and sampling from graphical models and energy-based models.
Recently, denoising diffusion probabilistic models (DDPM) \citep{ho2020denoising} and 
score matching with Langevin dynamics (SMLD) with its variants 
\citep{song2019generative, song2020improved, song2020score} have achieved the state-of-the-art results on data generation \citep{dhariwal2021diffusion,chen2020wavegrad,niu2020permutation,hoogeboom2021argmax}.

Standard diffusion processes used in ML 
can be classified into two categories: %
1) \emph{infinite (or mixing) time} diffusion processes such as Langevin dynamics, 
which requires the process to run sufficiently long to converge to the \emph{invariant distribution}, 
whose property is leveraged for the purpose of learning and inference; 
and 2) \emph{fixed time diffusion} processes such as 
DDPM, SMLD, and Schrodinger bridges \citep{de2021diffusion}, 
which are designed to output the desirable results 
at a pre-fixed time. 
Although fixed-time diffusion has been show to surpass 
infinite time diffusion on both speed and quality, it still yield slow speed for modern applications due to the need of a pre-specified time and the incapability to adapt the time based on the difficulty of instances and problems. 
Moreover, standard diffusion models 
are naturally designed on $\RR^d$, and can not work for discrete and structured data without special cares. 

In this work, we study and explore a different \emph{first hitting time} diffusion model that terminates at the first time as it hits a given domain, and leverages the distribution of the exit location (known as exit distribution, or harmonic measure \citep{oksendal2013stochastic}) as a tool for learning and inference. We provide the basic framework and tools for first hitting diffusion models. We leverage our framework to develop a general approach for learning deep generative models based on first hitting diffusion. This approach generalizes SMLD and its SDE extensions but can be attractively applied to a range of discrete and structured domains. This contrasts with the standard diffusion models, which are restricted to continuous $\RR^d$  data. In particular, we instantiate our framework to three cases, yielding new diffusion models for learning 1) spherical, 2) binary and 3) categorical data. 
In addition, the proposed diffusion model gives different instances adaptive arrival times and can generate high-quality samples using fewer diffusion steps. We discuss theoretical properties and fast implementation of our methods and demonstrate their practical efficiency in a suite of practical learning problems.

\begin{figure}
\centering
\includegraphics[width=.9\textwidth]{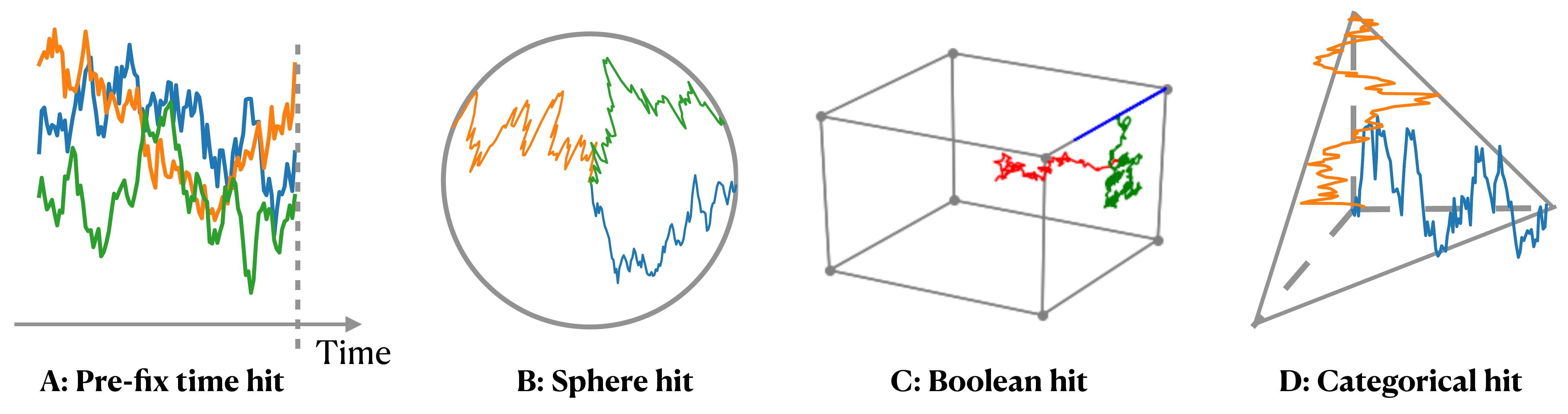}
\caption{{The four hitting schemes introduced in this paper. A: fixed-time hit, the process terminates at a fixed time; B: Sphere hit, hitting the boundary of a sphere from inside; C: Boolean hit, each coordinate terminates when it hits 0 or 1 and the whole process terminates when all of its coordinates terminate; D: Categorical hit, hitting the one-hot codes based on a conditioned process.}}
\label{fig:4hit}
\end{figure}

\section{Main Framework}
\subsection{First Hitting Diffusion Processes} 
Let $\tg$ be a distribution of interest on a domain $\Omega \subset \RR^d$.  
The goal is to construct a \emph{first hitting stochastic process}, which starts from a point outside of $\Omega$ and returns a sample drawn from $\tg$ when it first hits set $\Omega$. 
We start with introducing the new first hitting model.

Let $\traj \X \defeq \{\X_t \colon t \in[0,+\infty)\}$ be a continuous-time Markov process  with probability law $\Q$ taking value in a set $V$ that contains $\Omega$ as a subset.  
Here $\Q$ is a probability measure 
defined on the space of all continuous trajectories $C([0,+\infty), ~ \RR^d)$. We use $\Q_t$ to denote the marginal distribution of $\X_t$ at time $t$. 
We assume that the process is initialized from a point $\X_0$ outside of $\Omega$. 
Denote by $\tau$ the \emph{first hitting time}  of $\X_t$ on $\Omega$, that is, 
$\tau = \inf_{t} \{ t \geq 0 \colon \X_t \in \Omega  \}.$ 
We call that 
$\X_t$ is absorbing to set $\Omega$ if %

i) The process  enters $\Omega$ in finite time almost surely when initialized from anywhere in $V$, 
that is, $\Q(\tau < +\infty~|~ \X_0= z) =1$, $\forall z \in V$.

ii) The process stops to move once it arrives at $\Omega$, that is,   $\Q(\X_{t+s} = \X_t ~|~ \X_t \in \Omega) = 1$, $\forall s, t \geq 0$. 

We define the \emph{Poisson kernel} of $\Q$ as the conditional distribution of $\X_\tau$ given $\X_t = z$, denoted by $\Q_\Omega(\d\x  ~|~ \X_t=z) \defeq \Q(\X_\tau = \d\x ~|~ \X_t= z)$.  
The marginal distribution of $\X_\tau$, which we 
write as $\Q_\Omega(\df \x) = \Q(\X_\tau = \df \x)$, is called the 
\emph{exit distribution, or harmonic measure}.  Note that 
$\Q_\Omega(\df \x)= \int_V \Q_\Omega(\df x~|~\X_0 = z) \Q_0(\d z)$.
The crux of our framework is to leverage the exit distribution $\Q_\Omega$ as a tool for statistical learning and inference, 
which is different from traditional frameworks that exploit the properties of the distributions at a fixed time or at convergence.

\begin{exa}[Sphere Hitting]\label{exa:spherical}
As shown in Figure~\ref{fig:4hit}-B, 
let $V = \{x\in \RR^d \colon \norm{x}\leq 1\}$ be the unit ball and $\Omega  =  S_{d} \defeq \partial V$ the unit sphere. 
Let  $\X$ be a Brownian motion starting from $z \in V$ and stopped once it hits the boundary $\Omega$. It is written as 
\bbb \label{equ:sphere1}
{\dist Q^{\sph}}: &&
\df \X_t = \ind(\norm{\X_t}<1) \df W_t, 
~~~ ~~ \X_0  \in V,  
\eee 
where $W_t$ is a Wiener process; 
 the indicator function $\ind(\norm{\X_t}<1)$ 
sets the velocity to zero and hence stops the process once $\X_t$ hits $\Omega$. 
The Poisson kernel in this case is a textbook result: 
\bbb \label{equ:PoissonSd}
\Q^{\sph}_\Omega(\df x ~|~\X_t = z) \propto 
\frac{1-\norm{z}^2}{\norm{x-\z}^{{d}}} \times  \mu_{\Omega}(\df z),
&& \text{where $\mu_\Omega$ is the surface measure on $\Omega=S_{d}$. 
}
\eee  
\end{exa} 

\begin{exa}[Boolean Hitting] \label{exa:boolean}
As shown in Figure~\ref{fig:4hit}-C, let $V = [0,1]^d$ be the unit cube and
$\Omega = B_d \defeq \{0, 1\}^d$ the  Boolean cube. 
Let $\X$ be a Brownian motion starting from $Z_0\in V$ and confined inside the cube $V$ in the following way:   %
\bb 
\Q^{\binary}: &&
\df \X_{t,i} = \ind(\X_{t,i}\in(0,1))  \df W_{t,i}, ~~~ \forall i \in\{1,2,\cdots,d\}, %
\ee 
where $\X_{t,i}$ is the $i$-th element of $\X$. 
Here, each coordinate $\X_{t,i}$ stops to move once  it hits one of the end points ($0$ or $1$). 
It can be viewed as a particle flying in a room that sticks on a wall once it hits it. 
\end{exa} 
\begin{theoremEnd}[\isproofhere]{pro}
\label{thm:ber}
The Poisson kernel of $\Q^\binary$ is a simple product of Bernoulli distributions: 
$$
\dist Q^\binary_\Omega( x ~|~ Z_t = z) 
=\ber(x | z) \defeq \prod_{i=1}^d \ber(x_i | z_i),~~~\text{where}~~~
\ber(x_i|z_i) = x_i z_i + (1-x_i) (1-z_i); 
$$
$\ber(x_i|z_i)$ is the likelihood function of observing $x_i\in\{0,1\}$ under 
Bernoulli$(z_i)$ with %
$z_i\in[0,1]$. 
\end{theoremEnd}
\begin{proofEnd}%
As $\Q^{B_d}$ is a product of identical and independent one-dimensional processes, it is sufficient to consider the one dimension case $d=1$, in which case the process is a Brownian motion $ \d Z_t =\d W_t$ starting from interval $Z_0 \in [0,1]$ and stopped $\tau=\min_{t} \{t\colon ~Z_t \not\in (0,1)\}$ when it exits the interval. Hence, the Poisson kernel is 
$$
\dist Q^\binary_\Omega( x ~|~ z) 
= \prob(W_\tau =x ~|~ W_{t} = z) 
= \prob(W_\tau =x ~|~ W_{t} = z),~~~~\forall x\in \{0,1\}
$$
Then, 
it is a textbook result that $\prob(W_\tau =x ~|~ W_{t} = z) = x z + (1-x) (1-z) = \ber(x|z).$ See e.g., Eq.~3.0.4, Page 212 of \citet{borodin2015handbook}. 
\end{proofEnd}

\begin{exa}[Fixed Time Hitting]\label{exa:fixedtime}
Our first hitting framework includes the more standard models with fixed terminal time. To see this, let $ \bar \Z_t = (t, \Z_t)$ be a stochastic process $\X_t$  with law $\Q$ augmented with time $t$ as one of its coordinates. Let $V= [0,\t] \times \RR^d$ and $\Omega = \{\t\} \times \RR^d$, where $\Omega$ is a vertical plane on the augmented space. Then  the hitting time $\tau$ equals $\t$ deterministically, and the exit distribution equals the marginal distribution of $\Z_\t$ at time $\t$. See Figure~\ref{fig:4hit}-A, for illustration.
\end{exa}

 \subsection{Diffusion Process Tools: Conditioning and $h$-transform} 
 
 We introduce some basic tools for diffusion processes, including  how to conduct conditioning, and exponential tilting (via $h$-transform) on diffusion processes. 
We apply these tools to the first hitting models we have. 
The readers can find related background in \citet{oksendal2013stochastic, sarkka2019applied}. 

Assume $\X$ is  a general Ito diffusion process in $V$ that is absorbed to $\Omega$, denoted as $\itoo(b, \sigma)$, 
\bbb 
\label{equ:baseX}
\Q \sim \itoo(b, \sigma):  
&&
\df \X_t = b_t(\X_t) \dt + \sigma_t(\X_t)\d W_t,~~~\forall t\in[0,+\infty), && \X_0 \sim {\Q_0},
\eee 
where $b_t(x)\in \RR^d$ is the drift term and 
$\sigma_t(x)\in \RR^{d\times d}$ is a positive definite diffusion matrix.  
We always assume that  $b$ and $\sigma$ are sufficiently  regular to yield a unique weak solution of \eqref{equ:baseX}.%

\paragraph{Conditioning} 
 A step in our work is to find the distribution of the trajectories of a process $\Q$ conditioned on a future event, e.g., the event of hitting a particular value $x$ at exit, that is, $\{\X_\tau = x\}$.  
A notable result is that the conditioned diffusion processes are also diffusion processes. %
Given a point $\x \in \Omega$ on the exit surface, 
the process of $\Q(\cdot ~|~ \X_\tau = \x)$ can be shown to be the law of the following diffusion process \citep{doob1984classical, sarkka2019applied}:  
\bbb \label{equ:Xtr} 
\Q(\cdot|\X_\tau = x): &&
\df \X_t
= \left (b_t(\X_t) + \blue{\sigma_t^2(\X_t) \dd_{\X_t} \log q_{\Omega}(x ~|~\X_t) } \right ) \dt + \sigma_t(\X_t) \df W_t, 
~~
\X_0 \sim \mu_{0|x},
 \eee  
where $q_{\Omega}(x~|~ z)$ is the density function  of the Poisson kernel $\Q_{\Omega}(\d x ~|~ \X_t=z)$ w.r.t. a reference measure $\mu_\Omega$ on $\Omega$,
  and $\sigma^2$ is the matrix square of $\sigma$,
  and the conditional initial distribution $\mu_{0|x} = \Q_0(\cdot ~|~\X_\tau=x)$ is the posterior probability of $\X_0$ given $\X_\tau = x$. 
  
  Intuitively, 
the additional drift term $\dd_{\X_t} \log p_{\Omega}(x ~|~\X_t)$ plays the role of steering the process towards the target $x$, with an increasing magnitude as $\X_t$ approaches $\Omega$ (because $ P_{\Omega}(\cdot ~|~Z_t=z)$ converges to a delta measure centered at $x$ when $z$ approaches $\Omega$). 
 This process is known as a diffusion \emph{bridge}, 
 because it 
is guaranteed to achieve
$\X_\tau = x$  at the first hitting time 
with probability one. 

\begin{theoremEnd}[\isproofhere]{pro}\label{thm:Qsdconditioning} 
For $\Q^{\sph}$, the process conditioned on $\X_\tau = x \in S_{d}$ at exit is 
\bbb \label{equ:Xtrsphere} 
\dist Q^{\sph}(\cdot~|~ \X_\tau = x)
: &&&&&&&\df \X_t
= \ind(\norm{\X_t}<1)\left(\blue{ \dd_{\X_t}\log \frac{1-\norm{\X_t}^2}{\norm{x-\X_t}^{d}} } \dt + \df W_t\right). 
 \eee 
Here the additional drift term  (colored in blue)
 grows to infinity 
 if $\norm{Z_t} \to 1$ but $\norm{{Z_t}- x}$ is large, and hence enforces that $Z_\tau = x$ when we exit the unit ball. 
\end{theoremEnd}
\begin{proofEnd} 
It is a straightforward application of formula \eqref{equ:Xtr} in the case of $b_t = 0$, 
$\sigma_t(Z_t) = \ind(\norm{Z_t} <1)$ and $q_\Omega(x|z) = \frac{1-\norm{z}^2}{\norm{x-z}^d}$ as shown in \eqref{equ:PoissonSd}. 
\end{proofEnd}

\begin{theoremEnd}[\isproofhere]{pro}
\label{thm:binaryconditioning}
For $\Q^\binary$, the process conditioned on $\X_\tau = x \in \{0,1\}^d$ at exit is 
\bbb \label{equ:binaryBridge}
\Q^{\binary}(\cdot | Z_\tau = x): && 
\d Z_{t,i} = \ind(Z_{t,i}\in(0,1)) \left ( 
\blue{\frac{2x_i - 1}{x_i z_i +(1-x_i) (1-z_i)}}\dt +  \d W_{t,i} 
\right ),~~~\forall i.%
\eee  
The additional drift term (colored in blue) enforces that $\X_{\tau,i} = x_i$ at the exit time as the drift would be infinite if 
$z_i$ is still far from $x_i$ when $z_i$ is close to $\{0,1\}$. 
\end{theoremEnd}
\begin{proofEnd}%
It is a straightforward application of formula \eqref{equ:Xtr} in the case of $b_t = 0$, 
$\sigma_t(Z_t) = \mathrm{diag}(\ind(Z_t\in (0,1)))$ and $\Q_\Omega(x|z) =\ber(x|z)$ as shown in Proposition~\eqref{thm:ber}. 
\end{proofEnd} 
\begin{theoremEnd}[\isproofhere]{pro}
\label{thm:bb} 
For the fixed time diffusion 
in Example~\ref{exa:fixedtime}, let $\Q^{T}$ be the standard Brownian motion $\d Z_t = \d W_t$ stopped at a fixed time $t=T$, 
then $\Q$ conditioned on $\Q^T(Z | Z_T= x)$ is 
\bbb \label{equ:fixedTimeBridge}
\Q^T%
(\cdot | Z_\tau = x): && 
\d Z_{t} = \ind(t\le T) \left ( 
\blue{\frac{Z_{t}-x}{T-t}}\dt +  \d W_{t} 
\right ). 
\eee  
The additional drift (colored in blue) forces $Z_T=x$ as it grows to infinity if $Z_t \neq x$ while $t \to T$.
\end{theoremEnd}
\begin{proofEnd}%
This is the standard result on Brownian bridge. In particular, we just need to note that $(Z_{T} ~|~Z_t=z) \sim \normal(z, ~T-t)$ and apply formula \eqref{equ:Xtr}. 
\end{proofEnd}

\paragraph{$h$-Transform} Assume we want to modify the Markov process $\Z$ such that 
its exit distribution $\Q_\Omega$ matches the desirable target distribution $\tg$.  
Doob's $h$-transform \cite{doob1984classical} provides a simple general procedure to do so. 
Note that by disintegration theorem, we have $\Q(\d Z) = \int \Q_{\Omega}(\d x) \Q(\d \X ~|~ \X_\tau = x)$, 
which factorizes $\Q$ into the product of the exit distribution and the conditional process given a fixed exit location $Z_\tau =x$. 
To modify the exit distribution of $\Q$ to $\tg$,
we  can simply replace $\Q_\Omega$ with $\tg$ in the disintegration theorem, yielding %
\bbb \label{equ:ptg}
\Q^{\tg}(\d \X)
\defeq \int \tg(\d x) \dist Q(\d \X ~|~ \X_\tau= x)  
= \tgd(\X_\tau )\dist Q(\d \X),~~~\text{with}~~~ \tgd(\X_\tau )\defeq \frac{\d \tg }{\d \Q_\Omega}(\X_\tau), 
\eee 
where $\tgd=\frac{\d \tg }{\d \Q_\Omega}$ is the 
Radon–Nikodym derivative (or density ratio) between $\tg $ and $\Q_\Omega$, and 
$\Q^\tg$ is called an $h$-transform of $\Q$.      
Intuitively, 
$\Q^\tg$ is the distribution of trajectories 
$Z\sim \Q(\cdot|Z_\tau = x)$ when the exit location $x$ is randomly drawn from $x\sim \tg$. 
We can also view $\tgd(\X_\tau)$ as an importance score of each trajectory $\X$ based on its terminal state $\X_\tau$, and 
$\Q^{\tg}$ is obtained 
by reweighing (or tilting) the probability of each trajectory based on its  score. %

If $\Q$ is a diffusion process, then $\Q^\tg$ is also a diffusion process. In addition, $\Q^\tg$ is  the law of the following diffusion process: 
\bbb \label{equ:bridgeXh}
\Q^\tg: &&\df \X_t = \left (b_t(\X_t) + \blue{\sigma^2_t(\X_t) \dd_{z} \log h^\tg_t(\X_t) } \right ) \dt + \sigma_t(\X_t) \df W_t, ~~~ 
\X_0 \sim \Q^\tg_0
\eee
where the initial distribution $\Q_0^\tg$ and $ h^{\tg}$ in the drift term are defined as 
\bbb \label{equ:barh}
\Q_0^\tg(\d z) 
& = \int_{\Omega} \tgd(x) \Q(\X_\tau=\d x, \X_0=\d z) \\ 
 h^\tg_t(z) &  = \E_{\dist Q} [\tgd(\X_\tau)~|~ \X_t = z] = 
\int_{\Omega} \tgd(x)  \dist Q(\X_\tau =\d  x~|~ \X_t = z). 
\eee  
It is clear that $h$ coincides with $\pi^*$ on the boundary, that is, $ h_{\tgd}(x, t) = \tgd(x)$ for all $x\in \Omega, t\geq 0$. 
The name of $h$-transform comes from the fact that $h^\tg$ is a (space-time) harmonic function w.r.t. $\Q$ in the light of a mean value property: 
$h^{\tg}_t(z) = \E_{\Q} [ h^{\tg}_{t+s}(\X_{t+s})~|~ \X_t = z],$  $\forall s, t >0. $
   $\Q^\tg$ yields a simple variational representation in terms of Kullback–Leibler (KL) divergence.
  \begin{theoremEnd}[\isproofhere]{pro}[Variational Principle]\label{thm:var}
  The $\Q^\tg$ in \eqref{equ:ptg} yields 
 \bbb 
 \Q^\tg  & = \argmin_{\P \in \mathcal P(V, \Omega)} \left\{\KL(\P~||~\Q) \defeq \E_{\P}\left [ \log \frac{\df \P}{\df \Q}(Z)\right ], ~~~~s.t.~~~~ \P_\Omega = \tg \right\}\label{equ:sch}  \\    %
  & = \argmin_{\P \in \mathcal P(V, \Omega)}  \left\{\KL(\P~||~ \Q^\tg) \equiv \KL(\P~||~\Q) - \E_{\P}[\log  \tgd(\X_\tau)] \right\}, \label{equ:schr}
 \eee 
 where $\mathcal P(V, \Omega)$ denotes the set of path measures on $V$ that is absorbed to $\Omega$. 
 \end{theoremEnd}  
 Eq.~\eqref{equ:sch} shows that $\Q^\tg$ is the distribution with $\tg$ as the exit distribution that has the minimum KL divergence with $\Q$.
 It can be viewed as a Schrodinger half bridge problem \citep[e.g.,][]{pavon2021data}, which
 enforces the constraint of $\P_T = \tg$ at a fixed time $T$, rather than %
 the first hitting time $\tau$. 
   Eq.~\eqref{equ:schr} shows that the constraint can be turned into a penalty. %
 \begin{proofEnd} %
 Eq.~\eqref{equ:schr} 
 is the direct result of $\Q^\tg = \argmin_{\P } \KL(\P~||~ \Q^\tg)$, and that 
 $$\KL(\P~||~ \Q^\tg) \equiv \KL(\P~||~\Q) - \E_{\P}[\log  \tgd(\X_\tau)], 
 $$
 where we used $\Q^\tg(\d Z) = \Q(\d Z) \pi^*(Z_\tau)$. 

Eq~\eqref{equ:sch} is a simple consequence of the disintegration theorem. Note that any $\P$ that satisfies $\P_\Omega = \tg$ can be written into $\P (\d \X) = \tg(\d \X_\tau ) \P(\d \X ~|~ \X_\tau)$. By the chain rule of KL divergence, 
\bbb \label{equ:chainKL}
\KL(\P~||~\Q) = \KL(\tg~||~\Q_\Omega) 
+ \E_{\X_\tau\sim \tg }\left [ \KL(\P(
\cdot | \X_\tau) ~||~ \Q(\cdot|\X_\tau)) \right ]. 
\eee
Since it is constrained that $\P_\Omega = \tg$, the optimal $\P$ is determined by the choice of $\P(\cdot | Z_\tau)$ and it  should yield $\P(\d \X ~|~ \X_\tau) = \Q(\d \X ~|~ \X_\tau)$. Therefore, the optimal $\P$ is   $\tg(Z_\tau) \Q(\cdot|Z_{\tau}) = \Q^\tg$. 

In fact, by the same derivation, we can see that \eqref{equ:sch} remains correct if we replace $\KL(\P~||~\Q)$ with $\KL(\Q~||~\P)$.
 \end{proofEnd}  

\paragraph{First Hitting Diffusion for Sampling}
The $h$-transform above readily provides 
a first hitting diffusion approach to approximate sampling from $\tg$, assuming we can approximate the drift term $h^\tg$. The Schrodinger-Follmer sampler \citep{huang2021schrodinger} can be viewed as a special case of this approach with a fixed exit time.  
We leave further exploration to future works. See more discussion in Appendix \ref{apx:sample}.

\subsection{Learning First Hitting Diffusion Models}
\begin{figure*}[t]
    \centering
    \includegraphics[width=1.01\textwidth]{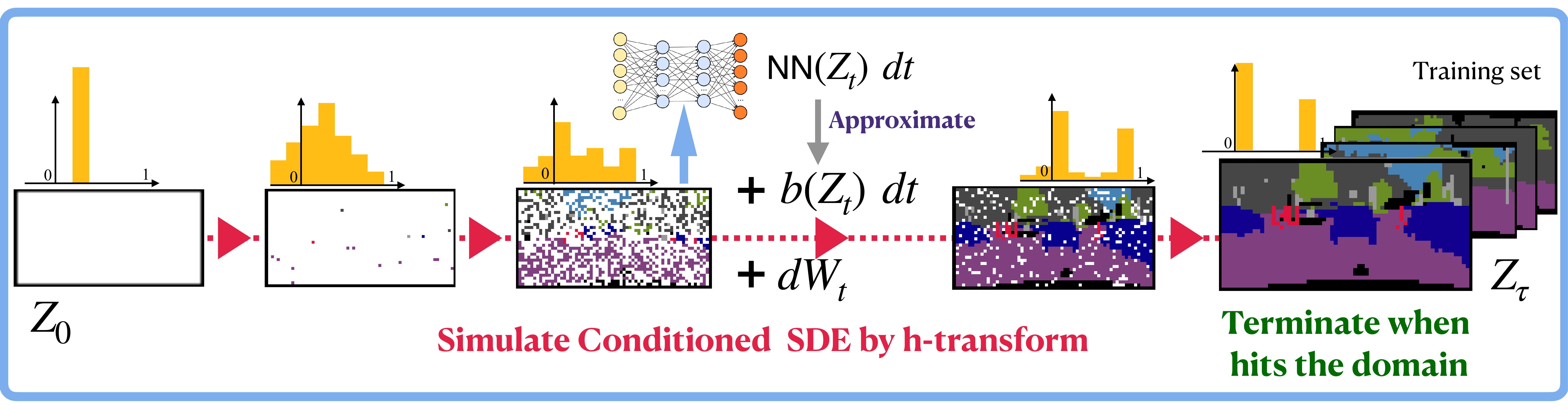}
    \vspace{-0.6cm}
    \caption{\small{The training pipeline of FHDM. Start from initial distribution, we use h-transform to simulate a conditioned SDE such that the process terminates at the desired destination data from training set at its hitting time. The network is trained to approximate the drift term ($b(Z_t)$), resulting a score-matching loss that is equivalent to the KL divergence.}} 
    \label{fig:pipeline}
\end{figure*}

\begin{algorithm}[t] 
\caption{Learning Generative Models by First Hitting Diffusion} \label{alg:learningmainmain}
\begin{algorithmic}
\STATE \textbf{Inputs \& Goal}: A data $\hat \Pi \defeq \{x\datai\}$ drawn from $\tg$ on $\Omega$. %
A baseline process $\Q$ and a  model $\P^\theta$ that are absorbing to $\Omega$. Want to find $\theta$ such that  $\P_\Omega^\theta \approx \tg$.  
\STATE \textbf{Training}: 
Approximate  $\hat \theta =\argmin_\theta \L(\theta)$ by stochastic gradient descent with batches of data (approximately) drawn $Z\sim \Q(\cdot |Z_\tau =x)$ and $x\sim \hat \Pi$.  
\STATE \textbf{Inference}: Simulate $\P^{\hat\theta}.$
\end{algorithmic}
\end{algorithm} 

We illustrate the learning pipeline of our First Hitting Diffusion Models (FHDM) in Figure~\ref{fig:pipeline}. Assume $\tg$ is unknown and we observe it through an i.i.d. sample $\{x\datai\}_{i=1}^n $ drawn from $\tg$. 
We want to fit the data with a parametric diffusion process $\itoo(s_\theta, \sigma)$ in $V$ that is absorbing to $\Omega$, 
\bbb \label{equ:ptheta} 
\P^\theta \colon &&&&&&\df 
\X_t =%
s_t^\theta(\X_t) \dt + 
\sigma_t(\Z_t)\d W_t, 
&& && ~~ \X_0 \sim \P^\theta_0, 
\eee  
such that the exit distribution $\P^\theta_\Omega$ matches the unknown $\tg$.  
Here $s^\theta_t(z)$ is a deep neural network with input $(z,t)$ and parameters $\theta$.
We should design $s^\theta$ and $\sigma$ properly to ensure the absorbing property. 
The standard approach to estimate $\tg$ is maximum likelihood estimation,
which can be viewed as approximately solving $\min_{\theta} \KL(\tg ~||~ \P^\theta_\Omega).$ However, calculating the likelihood of the exit distribution $\P^\theta_\Omega$ of a general diffusion process is computationally intractable. To address this problem, 
we fix $\Q$ as a ``prior'' process, 
and augment the data distribution $\tg$ to the $h$-transform $\Q^\tg$, whose exit distribution $\Q^\tg_\Omega$ matches $\tg$ by definition. 
Note that we can draw i.i.d. sample from $\Q^\tg$ in a ``backward'' way: first drawing an exit location $x\sim \tg$ from the data, 
and then draw the trajectory $Z$ from $\Q(\cdot | Z_\tau = x)$ with the fixed exit point.  To train a generative model, we train $\P^\theta$ to fit it with the data drawn from $\Q^\tg$ by maximum likelihood estimation:  
$$
\min_{\theta} 
\left \{ \L(\theta)\defeq  \KL(\Q^{\tg}~||~ \P^\theta)\equiv 
-\E_{Z\sim \Q^\tg}\left  [\log p^\theta(Z) \right] + \const,  
\right\},$$
where $p^\theta = \frac{\d\P^\theta}{\d \Q^\tg}$ is Radon–Nikodym density function of $\P^\theta$ relative to $\Q^\tg$. 
By the chain rule of KL divergence in 
\eqref{equ:chainKL} in Appendix \ref{apx:proof}, 
we have  $\KL(\tg ~||~ \P^\theta_\Omega) \leq \KL(\Q^{\tg}~||~ \P^\theta)$. 
Therefore, if %
minimizing the KL divergence allows us to achieve $\P^\theta \approx \Q^\tg$, 
we should also have $\P_{\Omega} ^\theta \approx \Q_\Omega^\tg = \tg$.

Using Girsanov theorem \citep{liptser1977statistics}, we can calculate the density function $p^\theta$ and hence the loss function. 
\begin{theoremEnd}[\isproofhere]{pro}
\label{thm:loss}
Assume $\Q$ in \eqref{equ:baseX},  
and $\P^\theta$ in \eqref{equ:ptheta} are absorbing to $\Omega$. We have
 \bbb \label{equ:mainloss} 
  \L(\theta) 
 & =  \frac{1}{2} \E_{\Q^\tg} 
 \!\!\!\left [ \int_0^\tau  \norm{
 \sigma_t(\X_t)^{-1}(s^\theta_t(\X_t) -
 b_t(\X_t~|~ \X_\tau)) %
 }^2 \df t - \log p_0^\theta(\X_0) \right ] + \const, 
\eee
where $b_t(z | x) \defeq b_t(z) + \sigma_t^2(z)\dd_z \log p_{\Omega}(x| z)$ is the drift of the conditioned process $\Q(\cdot |Z_\tau = x)$ in \eqref{equ:Xtr}, and 
$p_0^\theta$ is the probability density function of the initial distribution $\P^\theta_0$. 
In addition, $\theta^*$ achieves the global minimum of $\L(\theta)$ if 
\bb 
s_t^{\thetat}(z) =  \E_{Z\sim \Qt}[b_t(z | Z_\tau)~|~ Z_t = z], 
&& 
\P_0^{\thetat} = \Qt_0 = \E_{x\sim \tg}[\Q_0^x(\cdot)]. 
\ee 
\end{theoremEnd} 
\begin{proofEnd} 
Denote by $\Q^x = \Q(\cdot | Z_\tau = x)$. 
Let $p^\theta$ be the density function of $\P^\theta$ w.r.t. to some reference measure (e.g., $\Q^\tg$). We have 
\bb 
\KL(\Q^\tg ~||~ \P^\theta) 
& =  -\E_{Z \sim \Q^\tg} [\log p^{\theta}(Z)] + \const \\
& = - \E_{x\sim \tg} \E_{Z\sim \Q^x} [\log p^{\theta}(Z)]  + \const \\ 
& = \E_{x\sim \Pi^*}\left [ \KL(\Q^x ~||~ \P^\theta) \right ] +\const, 
\ee 
where $ \KL(\Q^x ~||~ \P^\theta)$ can be evaluated using Girsanov theorem, 
\bb 
 \KL(\Q^x ~||~ \P^\theta)  
 & = \KL(\Q^x_0 ~||~ \P^\theta_0) 
 + \frac{1}{2}  \E_{Z\sim \Q^x}
 \left [\int_0^\tau 
\norm{s^\theta_t(Z_t) - b_t(Z_t|x)}^2_2  \dt 
 \right ] \\ 
 & = 
   \E_{Z\sim \Q^x}\left [  
 -\log p_0^\theta(Z_0) + 
 \frac{1}{2} \int_0^\tau 
\norm{s^\theta_t(Z_t) - b_t(Z_t|x)}^2_2  \dt 
 \right ]  + \const.  
\ee 
Hence 
\bb 
\L(\theta)
& =  \E_{x\sim \tg, Z\sim \Q^x}\left [  
 -\log p_0^\theta(Z_0) + 
 \frac{1}{2} \int_0^\tau 
\norm{s^\theta_t(Z_t) - b_t(Z_t|x)}^2_2  \dt 
 \right ]  + \const  \\ 
 & =  \E_{Z\sim \Qt}\left [  
 -\log p_0^\theta(Z_0) + 
 \frac{1}{2} \int_0^\tau 
\norm{s^\theta_t(Z_t) - b_t(Z_t|Z_\tau)}^2_2  \dt 
 \right ]  + \const . 
\ee 
\end{proofEnd}
Therefore, the optimal drift term $s_t^{\thetat}$ 
should match the conditional expectation of $b_t(z|x)$ with $x\sim \Q_\Omega(\cdot| Z_t= z)$, which coincides with  the drift of $\Qt$ in \eqref{equ:bridgeXh}.  %
The initial distribution of $\P^\theta$ should obviously match the initial distribution of $\Q^\tg$. In practice, we recommend eliminating the need of estimating $\P^{\theta_0}$ by starting $\Q$ from a deterministic point $Z_0 = z_0$, 
in which case $\P^\theta$ should initialize from the same deterministic point. 
See Algorithm~\ref{alg:learningmainmain}.%

\paragraph{Learning Spherical Hitting Models}
Take $\Q = \Q^{\sph}$ %
in Example~\ref{exa:spherical}, 
we get a method for learning generative models for data on the unit sphere. 
We set the model to be $\d Z_t = \ind(\norm{Z_t}<1)(f^\theta_t(Z_t) \dt  +  \d W_t)$ to ensure that it is absorbing to $S_d$. 
The loss function is 
\bb 
  \L(\theta) 
 & =  \frac{1}{2} 
 \E_{\substack{x\sim \tg \\ Z\sim \Q^x} }
 \!\!\!\left [ \int_0^\tau  \norm{
 f^\theta_t(\X_t) -
 \dd_{Z_t} \log \frac{1-\norm{Z_t}^2}{\norm{x - Z_t}^d}
 }^2 \df t - \log p_0^\theta(\X_0) \right ]  + \const. 
\ee 

\paragraph{Learning Boolean Hitting Models}
Taking $\Q=\Q^{\binary}$ as in Example~\ref{exa:boolean} 
provides an approach to learning diffusion generative models for binary variables. 
We  set the model $\P^\theta$ to be $\d Z_t = \ind(Z_t\in(0,1)) \circ(f_t^\theta(\theta) \dt + \d W_t)$ to ensure that $\P^\theta$ is absorbing to $B_d$ like $\Q^\binary$, where $\circ$ denotes element-wise multiplication. 
The loss function is 
\bb 
  \L(\theta) 
 & =  \frac{1}{2} 
 \E_{\substack{x\sim \tg \\ Z\sim \Q^x} }
 \!\!\!\left [ \int_0^\tau  
 \norm{
\ind(Z_t\in(0,1))\circ  \left (f^\theta_t(\X_t) - \dd_{Z_t}\log {\mathrm{Ber}(Z_t|x)} \right )
 }^2 \df t - \log p_0^\theta(\X_0) \right ]  + \const. 
\ee

\paragraph{Learning Fixed Time Diffusion Models}
Following the fixed time setting in Example~\ref{exa:fixedtime}, 
we can recover the standard fixed time diffusion models for continuous data, such as SMLD and DDPM. 
In particular, a natural choice is to set $\Q$ to be an O-U process $\d Z_t = \alpha_t \Z_t \dt + \sigma_t \d W_t$ initialized from $Z_0 \sim \normal(0, v_0)$ where $\sigma_t\geq 0, v_0>0$. We show in Appendix \ref{apx:compare} that SMLD ($\alpha_t = 0$) and DDPM ($\alpha_t>0$) is recovered as the limit case when $v_0\to +\infty$. %

\subsubsection{Learning Categorical Generative Models}
In addition to the boolean hitting model, 
we provide here a first hitting framework for learning categorical data.  
In this case, the data domain  $\Omega$ is $C_{d,m}=\{e_1,\ldots, e_d\}^m$, where $e_i=[0,\ldots, 1,\ldots, 0]$ is the $i$-th one-hot (or basis) vector in $\RR^d$, so the data is a $m$-dimensional and $d$-categorical. 

It is less straightforward to construct a first hitting diffusion process that is absorbing to $C_{d,m}.$ 
We leverage the conditioning technique to achieve this.
We explain the idea with $m=1$, of which  the general case is a direct product. 
The key  observation is that the one-hot vectors $C_{d,1}$ is a subset of the boolean cube $B_d=\{0,1\}^d$.   
Hence, by definition, the conditioned process $\Q^{C_{d,1}}\defeq \Q^{B_d}(\cdot|Z_\tau \in C_{d,1})$ 
exits at $C_{d,1}$ from the inside of $B_d$. 
Using the method of $h$-transforms \cite{doob1984classical, sarkka2019applied},  
$\Q^{\Omega} \defeq \Q^{B_d}(\cdot|Z_\tau \in \Omega)$ for any $\Omega\subset B_d$ is the law of 
\bb 
\d Z_t = \ind(Z_t\in(0,1))  \circ 
\left (   \dd_z \log \ber(\Omega~|~Z_t) \dt + \d W_t
\right), %
&& ~~~~ 
\ber(\Omega~|~z )
\defeq \textstyle{\sum}_{e\in \Omega} \mathrm{Ber}(e~|~ z).
\ee 
Another challenge is to %
construct a parametric family of $\P^\theta$ that is absorbing to $C_{d,m}$, regardless of the value of $\theta$. 
The result below shows that this can be done by simply adding on top of $\Q^\Omega$ any bounded neural network drift term. 
\begin{theoremEnd}[\isproofhere]{pro} 
\label{thm:feq}
Let $V= [0,1]^d$ and $\Omega$ is any subset of $B_d = \{0,1\}^d$. 
Assume $f_t^\theta(z)$ is any bounded measurable function. 
Then the following process is guaranteed to hit $\Omega$ when it exits $V$: $$
\d Z_t = \ind(Z_t\in(0,1)) \circ 
\left (   
f_t^\theta(Z_t) ~+~  
\dd_{z}\log \mathrm{Ber}(\Omega~|~Z_t) \dt + \d W_t
\right), ~~~~ Z_0 \in (0,1)^d.
$$
\end{theoremEnd} 
\begin{proofEnd} 
Assume $\norm{f}_\infty \defeq \sup_{t\in[0+\infty), x\in [0,1]^d}\norm{f_t(x)}_2 < +\infty$. 
Consider the following two processes 
with the same initialization: 
\bbb\label{equ:q0} \begin{split} 
& \Q^0 \colon \d Z_t = \ind(Z_t\in(0,1)) \circ 
\left (   
\dd_{z}\log \mathrm{Ber}(\Omega~|~Z_t) \dt + \d W_t
\right) \\
& \Q^f \colon \d Z_t = \ind(Z_t\in(0,1)) \circ
\left (   
f_t^\theta(Z_t) ~+~  
\dd_{z}\log \mathrm{Ber}(\Omega~|~Z_t) \dt + \d W_t
\right). 
\end{split}
\eee 
Girsanov theorem shows that $\KL(\Q^0~||~ \Q^f) = \frac{1}{2}\E_{\Q^0}[\int_0^\tau \norm{f_t (Z_t)}^2] \leq \frac{1}{2}\norm{f}^2_\infty \E_{\Q^0}[\tau] <+\infty$, 
where we used the fact that the expected hitting time  $\E_{\Q^0}[\tau]$ of $\Q^0$ is finite (see Lemma~\ref{lem:hehq0} below). 

Now $\KL(\Q^0~||~ \Q^f)<+\infty$ implies that that $\Q^0$ and $\Q^f$ has the same support. Hence the fact that $\Q^0$ guarantees to hit $\Omega$ when exiting $V$, i.e., $\Q^0(Z_\tau \in \Omega) =1$, 
when exit implies that $\Q^f$ has the same property, i.e., $\Q^f(Z_\tau \in \Omega) =1$. 
\begin{lem}\label{lem:hehq0}
Let $\tau^0 = \inf\{ t\colon Z_t \in \Omega\}$ be the first hitting time to $\Omega \subseteq \{0,1\}^d$ of the 
 the process $\Q^0$ in Eq.~\eqref{equ:q0}. Then $\E[\tau^0]<+\infty$. 
\end{lem}
\begin{proof}
Consider the following two processes 
starting from the same deterministic initialization $Z_0  = Y_0 = z_0\in (0,1)^d$: 
\bb
& \Q^0 \colon \d Z_t = \ind(Z_t\in(0,1)) \circ 
\left (   \dd_{z}\log \mathrm{Ber}(\Omega~|~Z_t) \dt + \d W_t \right) \\
& \Q^* \colon \d Y_t = \ind(Y_t\in(0,1)) \circ 
\left (  \d W_t \right).
\ee
Denote by $\tau^0$ and $\tau^*$ the corresponding hitting times to $\Omega$, that is,  $\tau^0 = \inf \{t \colon Z_t \in \Omega\}$, 
and $\tau^* = \inf \{t \colon Y_t \in \Omega\}$.  

Then we know from $h$-transform that $\Q^0$ is the conditioned process of $\Q^*$ given that $Z_\tau \in \Omega$, that is, 
$
\Q^0 = \Q^*(\cdot ~|~ Z_\tau \in \Omega)$.

Therefore, the first hitting time $\tau^0$ of $\Q^0$ has the same law as that of $\tau^* ~|~ Y_\tau \in \Omega$, that is, $\Q^0 (\tau^0 \in A) = \Q^*(\tau^* \in A ~|~ Y_\tau \in \Omega)$ for any measurable set $A\subseteq [0, +\infty)$. 

But we know that $\E[\tau^*~|~ Y_\tau \in \Omega] <+\infty$ due to the diffusion nature of Brownian motion. Hence, $\E[\tau^0] = \E[\tau^*~|~ Y_\tau \in \Omega] < +\infty$. 
\end{proof}
\end{proofEnd}
See Appendix \ref{apx:cate} for the summary of the algorithm for learning categorical data.

\subsubsection{Fast Sampling of Bridges} \label{sec:fast}  

One main step in calculating the loss $\L(\theta)$ is to draw trajectory $Z$ from the bridge $\Q^x = \Q(\cdot ~|~ Z_\tau = x).$ 
This can be achieved by simulating the bridge processes
using Euler–Maruyama method. 
This is not computationally costly because 
it is the simulation of elementary SDEs and does not involve deep neural networks.  
However, it does cause a slow down in the training algorithm  if the data $x$ is very high dimensional and the data size is very large. 

\begin{wrapfigure}{r}{.35\textwidth} 
    \centering
    \includegraphics[width=0.35\textwidth]{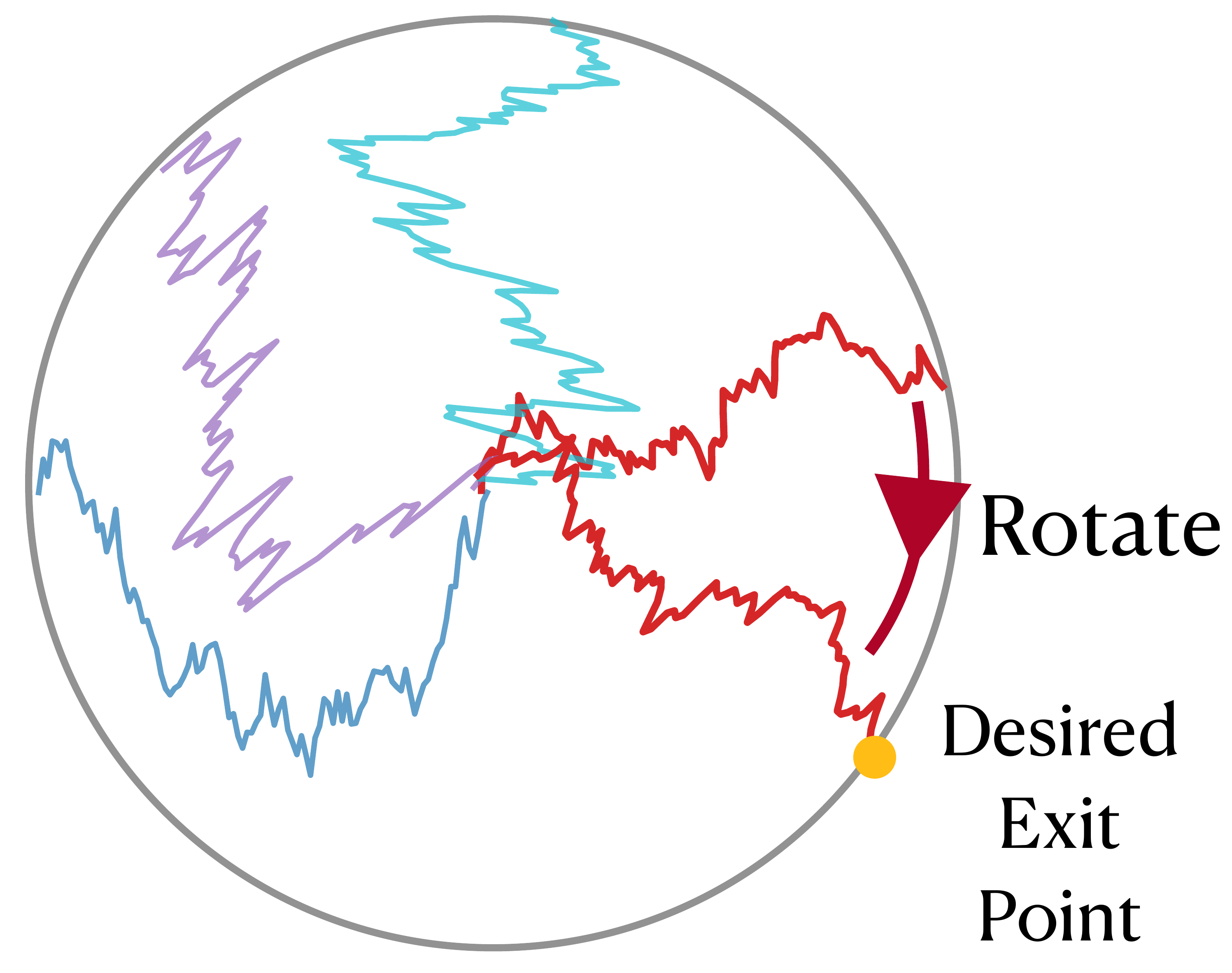}
\vspace{-2\baselineskip}
    \caption{ \small{To sample a conditioned process, we can pick up a trajectory of unconditioned process and rotate it so that it exits at a given point.}} 
    \label{fig:rot}
\end{wrapfigure}

To speed up the training, we propose a fast algorithm for simulating bridges by exploiting the symmetry 
when we initialize from a point $z_0$  (e.g., the center of sphere) 
around which $\Omega$ and $\Q$ are  rotational symmetric.  
 The idea is simple: 
we simulate the \emph{unconditioned} process $\Q$ to get a trajectory $Z$ that exits at any point. Then, to obtain the conditional process $\Q(\cdot | Z_\tau = x)$, 
we simply rotate the trajectory $Z$ such that the exit point $Z_\tau$ is transformed from the original one to $x$.  
An advantage is that we can pre-simulate a large number of trajectories before training and only need to apply the rotation operator to get specific conditioned processes during training. Figure~\ref{fig:rot} gives an illustration using the example of sphere hit. The idea can be applied similarly for other types of hitting.

\begin{pro}
Assume $Z$ with law $\Q$ initialized from $z_0\in \RR^d$ is absorbing to $\Omega$. 
For $x,x'\in \Omega$, 
let $\mathtt{rot}_{x'\to x}$ be the rotation operator around $z_0$  that transforms $x'$ to $x$ (hence  $\mathtt{rot}_{x'\to x}(x') =x$).
Assume that $\Omega$ and $\Q$ are rotation invariant 
around $z_0$ in that $\mathtt{rot}_{x',x}(\Omega) = \Omega$ and $\mathtt{rot}_{x',x}(Z)\sim \Q$ when $Z\sim \Q$ for any $x',x\in \Omega$. 
Then if $Z\sim \Q$, we get $Z' = \mathtt{rot}_{Z_\tau \to x} (Z)$, a sample drawn from $\Q(\cdot | Z_\tau = x).$
\end{pro}

Such a fast sampling approach is applicable for all the categorical, sphere and binary distributions.

\subsection{Discretization Error}
In practice, the Euler-Maruyama method is applied to discretize the process. Analyzing the discretization error of a random hitting process is more difficult than that of a fixed-time process. In a fixed-time process, both discretized and continuous processes terminate at the same time, making the coupling tricks applicable for analyzing the discretization error based on the $\ell_2$ Wasserstein distance. Standard analysis under Lipschitz continuity assumption of the drifts gives $O(\Delta)$ error rate where $\Delta$ is the discretization step size \citep{sarkka2019applied}. In comparison, the key challenge of analyzing the FHDM is that the discretized and continuous processes may not terminate at the same time, and 
thus we need to bound the probability of the difference of the hitting time distribution in the analysis. Besides, in practice, we might also apply some time truncation tricks in order to have a bounded waiting time for generating. In Appendix \ref{apx:discret}, we provide a full analysis and show that FHDM also yields $O(\Delta)$ discretization error asymptotically.

\section{Related Work}
\paragraph{Diffusion Generative Model on Different Domains}
Diffusion generative model has been demonstrated to be powerful in generation of general continuous data such as image \citep{sohl2015deep, song2019generative, ho2020denoising, song2020improved, song2020score, dhariwal2021diffusion}, point cloud shape \citep{cai2020learning,luo2021diffusion,zhou20213d} and audio \citep{chen2020wavegrad, kong2020diffwave}. Recently, diffusion generative model has also been extended to learn to generate data on special domains such graph \citep{niu2020permutation}, segmentation map \citep{hoogeboom2021argmax}, text \citep{hoogeboom2021argmax} and manifold data \citep{de2022riemannian}. Such a generalization of diffusion model is usually case-by-case and is based on applying constraints to ensure the data remains in the desired domain during the diffusion process \citep{hoogeboom2021argmax,de2022riemannian} or use heuristic approximation to round the data into the discrete space \citep{niu2020permutation}. Our FHDM gives a unified  framework for generating data on special domain via a completely new mechanism of first hitting. 

\paragraph{Theoretical Framework on Diffusion Process}
Most existing diffusion models are based on the framework of time-reversing \citep{song2020score} in which the generation (i.e. denoising) process is learned based on its time-reversed stochastic differential equation trajectory that can be simulated easily, ignoring the mismatch of the initial distribution. In comparison, our framework is conceptually simpler and is only based on a forward process, in which the learning is based on conditioned stochastic differential equations (i.e., bridge) that can be simulated via $h$-transform. A similar framework is independently explored in~\citep{peluchetti2021non} but our method is more general and exploits the idea of first hitting. Schrodinger bridges is an another well studied framework of diffusion model \citep{wang2021deep,de2021diffusion,peluchetti2021non,chen2021likelihood}. However, using Schrodinger bridges usually require expensive forward-backward algorithms. 
It is also unknown whether or how Schrodinger bridges can be applied for generating data in special domains.

\section{Experiments}
We applied FHDM to distributions on various domains such as point cloud (general continuous distribution), distribution of climate and geography events on earth (continuous distribution on the sphere), unweighted graphs (distribution of binary matrices), and segmentation map of 2D image (high dimension categorical distribution). We demonstrate that

$\bullet$     1. As a generalization of the fixed-time processes such as DDPM, the fixed-time scheme of FHDM is a generative model of higher quality for general continuous distribution (section \ref{sec:pointcloud}).
     
$\bullet$     2. As a versatile model, FHDM is able to learn the distribution in many different domains and it outperforms existing specifically designed generative models (see section \ref{sec:exp_sphere}).
     
$\bullet$     3. The hitting time of FHDM is well-bounded and in several tasks, FHDM even requires much fewer diffusion steps than existing methods while generating  higher quality samples (see section \ref{sec:analysis}).

Besides, we also conduct experiments to understand the intuition of the first time hitting mechanism (section \ref{sec:analysis}) and demonstrate the acceleration of the fast sampling approach introduced in Section \ref{sec:fast} (see in Appendix \ref{apx:exp_res}). We include the visualization of the generated samples in Appendix \ref{apx:vis}. Please find the code at \url{https://github.com/lushleaf/first_hitting_diffusion}.
\subsection{Generation Experiment}
\paragraph{Point Cloud Generation} \label{sec:pointcloud}
Following \citet{luo2021diffusion}, we employ the ShapeNet dataset \citep{chang2015shapenet} to evaluate the generated point cloud. We compare our approach against several the state-of-the-art generative models including PC-GAN \citep{achlioptas2018learning}, GCN-GAN \citep{valsesia2018learning}, Tree-GAN \citep{shu20193d}, PointFlow \citep{yang2019pointflow}, ShapeGF \citep{cai2020learning} and DPM \citep{luo2021diffusion}. See Appendix \ref{apx:exp_det} for training details.
Following \citet{cai2020learning,luo2021diffusion}, we use minimum matching distance (MMD) and the coverage score (COV) paired with Chamfer distance as well as 1-NN classifier accuracy and the Jenson-Shannon divergence (JSD) to evaluate the quality of the generated point cloud. We refer readers to Appendix \ref{apx:exp_det} for more details on the metrics. Same to \citet{cai2020learning, luo2021diffusion}, we evaluate the quality on two categories, Airplane and Chair and the generated and reference point clouds are normalized into a bounding box of $[-1,1]^3$ at evaluation. Table \ref{tbl:pc} summarizes the results showing that \ours achieves the best performance on most criterion. 

\begin{table}
\centering{}%
\begin{tabular}{l|cccc|cccc}
\specialrule{.15em}{.07em}{.07em} 
\multirow{2}{*}{Model} & \multicolumn{4}{c|}{Airplane} & \multicolumn{4}{c}{Chair}\tabularnewline
\cline{2-9} \cline{3-9} \cline{4-9} \cline{5-9} \cline{6-9} \cline{7-9} \cline{8-9} \cline{9-9} 
 & MMD$\downarrow$ & COV$\uparrow$ & 1-NNA$\downarrow$ & JSD$\downarrow$ & MMD$\downarrow$ & COV$\uparrow$ & 1-NNA$\downarrow$ & JSD$\downarrow$\tabularnewline
\hline 
PC-GAN \citep{achlioptas2018learning} & 3.819 & 42.17 & 77.59 & 6.188 & 13.436 & 46.23 & 69.67 & 6.649\tabularnewline
GCN-GAN \citep{valsesia2018learning} & 4.713 & 39.04 & 89.13 & 6.669 & 15.354 & 39.84 & 77.86 & 21.71\tabularnewline
Tree-GAN \citep{shu20193d} & 4.323 & 39.37 & 83.86 & 15.646 & 14.936 & 38.02 & 74.92 & 13.28\tabularnewline
PointFLow \citep{yang2019pointflow} & 3.688 & 44.98 & 66.39 & 1.536 & 13.631 & 41.86 & 66.13 & 12.47\tabularnewline
ShapeGF \citep{cai2020learning} & 3.306 & \pmb{50.41} & \pmb{61.94} & 1.059 & 13.175 & 48.53 & \pmb{56.17} & 5.996\tabularnewline
DPM\citep{luo2021diffusion} & \pmb{3.276} & 48.71 & 64.83 & 1.067 & 12.276 & 48.94 & 60.11 & 7.797\tabularnewline
\rowcolor{lightgray}Ours & 3.350 & \pmb{50.41} & 67.21 & \pmb{0.986} & \pmb{6.644} & \pmb{49.50} & 56.87 & \pmb{5.913} \tabularnewline
\specialrule{.15em}{.07em}{.07em} 
\end{tabular}\caption{{Result of point cloud generation experiment. We adopt the base line from \citet{luo2021diffusion}. Bolded value indicates the best performance method.}} \label{tbl:pc}
\end{table}

\paragraph{Generating Distribution on Sphere} \label{sec:exp_sphere}
We apply \ours to generate distribution of occurrences of earth and climate science events on the surface of earth (which is approximated as a perfect sphere). Following \citet{de2022riemannian}, we consider 4 datasets: volcanic eruption \citep{NGDC_vol}, earthquakes \citep{NGDC_quake}, floods \citep{Brakenridge} and wild fires \citep{EOSDIS}. We compared \ours against the current the state-of-the-art baselines including Riemannian Continuous Normalizing Flows \citep{mathieu2020riemannian}, Moser Flows \citep{rozen2021moser}, mixture of Kent distributions \citep{peel2001fitting} and standard Score-Based Generative model on 2D plane followed by the inverse stereographic projection (Stereographic Score-Based) \citep{gemici2016normalizing} and Riemannian Generative Model \citep{de2022riemannian}.
Same to \citet{de2022riemannian}, we evaluate the method via the negative log-likelihood on the test set. We run our method for 5 independent trials and report the averaged metric with its standard deviation. We directly adopt the baseline result from \citet{de2022riemannian}. Table \ref{tbl:sphere} summarizes the result. See Appendix \ref{apx:exp_det} for additional details. 

\begin{table}
\begin{centering}
\begin{tabular}{l|llll}
\specialrule{.15em}{.07em}{.07em} 
 & Volcano & Earthquake & Flood & Fire\tabularnewline
\hline 
Mixture of Kent \citep{peel2001fitting} & $-0.80\pm0.47$ & $\ \ \ 0.33\pm0.05$ & $0.73\pm0.07$ & $-1.18\pm0.06$\tabularnewline
Riemannian CNF \citep{mathieu2020riemannian} & $-0.97\pm0.15$ & $\ \ \ 0.19\pm0.0.4$ & $0.90\pm0.03$ & $-0.66\pm0.05$\tabularnewline
Moser Flow \citep{rozen2021moser} & $-2.02\pm0.42$ & $-0.09\pm0.02$ & $0.62\pm0.04$ & $-1.03\pm0.03$\tabularnewline
Stereographic Score-based  \citep{gemici2016normalizing} & $-4.18\pm0.30$ & $-0.04\pm0.11$ & $1.31\pm0.16$ & $\ \ \ 0.28\pm0.20$\tabularnewline
Riemannian Score-based \citep{de2022riemannian} & $\pmb{-5.56\pm0.26}$ & $-0.21\pm0.03$ & $0.52\pm0.02$ & $\pmb{-1.24\pm0.07}$\tabularnewline
\rowcolor{lightgray}Ours & $-1.25\pm0.18$ & $\pmb{-0.27\pm0.02}$ & $\pmb{0.29\pm0.03}$ & $\pmb{-1.24\pm0.08}$ \tabularnewline
\specialrule{.15em}{.07em}{.07em} 
\end{tabular}\caption{{Result on generating distribution of occurrences of earth and climate science events on the surface of earth. Bolded value indicates the best method.}} \label{tbl:sphere}
\par\end{centering}
\end{table}

\begin{table}[t]
\centering{}%
\begin{tabular}{l|lllllllll}
\specialrule{.15em}{.07em}{.07em} 
{\multirow{2}{*}{Method}} & \multicolumn{4}{c}{Community-small} & \multicolumn{4}{c}{Ego-small} & \multirow{2}{*}{Avg.}\tabularnewline
\cline{2-9} \cline{3-9} \cline{4-9} \cline{5-9} \cline{6-9} \cline{7-9} \cline{8-9} \cline{9-9} 
 & Deg. & Clus. & Orbit. & Avg. & Deg. & Clus. & Orbit. & Avg. & \tabularnewline
\hline 
GraphVAE \citep{simonovsky2018graphvae} & 0.350 & 0.980 & 0.540 & 0.623 & 0.130 & 0.170 & 0.050 & 0.117 & 0.370\tabularnewline
DeepGMG \citep{li2018learning} & 0.220 & 0.950 & 0.400 & 0.523 & 0.040 & 0.100 & 0.020 & 0.053 & 0.288\tabularnewline
GraphRNN \citep{you2018graphrnn} & 0.080 & 0.120 & 0.040 & 0.080 & 0.090 & 0.220 & 0.003 & 0.104 & 0.092\tabularnewline
GNF \citep{liu2019graph} & 0.200 & 0.200 & 0.110 & 0.170 & 0.030 & 0.100 & \pmb{0.001} & 0.044 & 0.107\tabularnewline
EDP-GNN\citep{niu2020permutation} & 0.053 & 0.144 & 0.026 & 0.074 & 0.052 & 0.093 & 0.007 & 0.050 & 0.062\tabularnewline
\rowcolor{lightgray}Ours & \pmb{0.009} & \pmb{0.105} & \pmb{0.009} & \pmb{0.041} & \pmb{0.019} & \pmb{0.040} & 0.005 & \pmb{0.021} & \pmb{0.031} \tabularnewline
\specialrule{.15em}{.07em}{.07em}
\end{tabular}\caption{{Result on graph generation experiment. We report the averaged performance of our approach based on 5 independent runs, giving 0.0013 standard deviation of the averaged metric. The results of the other baselines are directly adopted from \citet{niu2020permutation}. Bolded value indicates the best method.}} \label{tbl:graph}
\end{table}

\begin{figure*}[t]
    \centering
    \includegraphics[width=\textwidth]{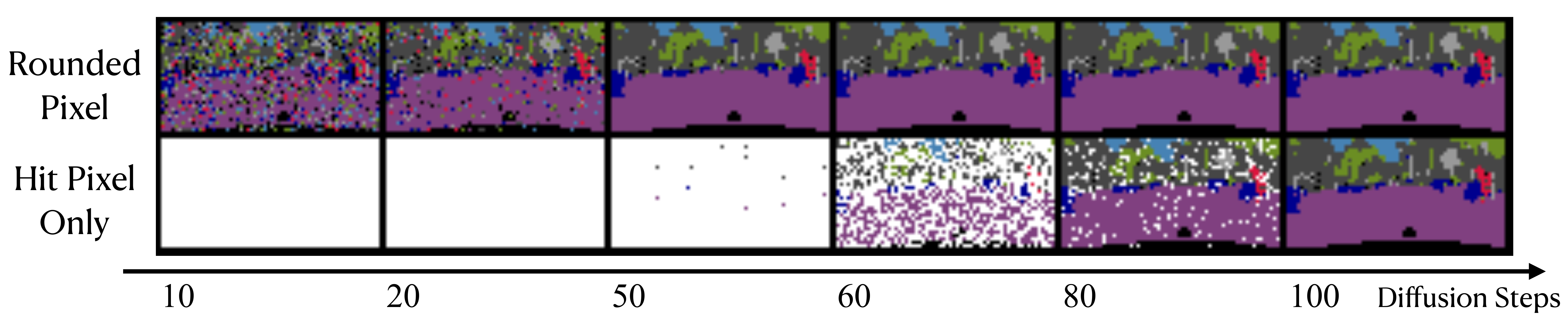}
    \vspace{-0.8cm}
    \caption{\small{The trajectory of generating a segmentation map image. The upper row shows the image where the category of all pixels are decided based on the argmax (i.e., rounding) of all the 8 scores. The lower row only plots the hit pixels of the snapshots.}} 
    \label{fig:hit_traj}
    \vspace{-0cm}
\end{figure*}

\paragraph{Graph Generation} \label{sec:exp_graph}
We apply \ours to generate (unweighted) graph that can be represented using binary adjacency matrix. Following the experiment setup in \citet{you2018graphrnn,liu2019graph,niu2020permutation}, we compare methods on two widely used benchmark datasets, Community-small and Ego-small. We apply the EDP-GNN \citep{niu2020permutation} that preserves the node permutation invariance to approximate the drift. We compare \ours against GraphRNN \citep{you2018graphrnn}, GNF \citep{liu2019graph}, GraphVAE \citep{simonovsky2018graphvae} and DeepGMG \citep{li2018learning}.
The maximum mean discrepancy (MMD) over three graph statistics (1. degree distribution; 2. cluster coefficient distribution; 3. the number of orbits with 4 nodes) proposed by \citet{you2018graphrnn} is used to evaluate the quality of the generative graphs. For our approach, we run 5 independent trails and report the averaged performance. See Appendix \ref{apx:exp_det} for additional training details. Table \ref{tbl:graph} summarizes the result, suggesting considerable improvement over the baselines.

\begin{wraptable}{r}{.53\textwidth}
\begin{centering}
\vspace{-1cm}
\begin{tabular}{l|ll}
\specialrule{.15em}{.07em}{.07em} 
Method & ELBO & IWBO\tabularnewline
\hline 
Round / Unif \citep{uria2013rnade} & 1.010 & 0.930\tabularnewline
Round / Var \citep{{ho2019flow++}} & 0.334 & 0.315\tabularnewline
Argmax / Softplus thres. \citep{hoogeboom2021argmax} & 0.303 & 0.290\tabularnewline
Argmax / Gumbel dist. \citep{hoogeboom2021argmax} & 0.365 & 0.341\tabularnewline
Argmax / Gumbel thres. \citep{hoogeboom2021argmax} & 0.307 & 0.287\tabularnewline
Multinomial Diffusion \citep{hoogeboom2021argmax} & 0.305 & -\tabularnewline
\rowcolor{lightgray}Ours & \pmb{0.066} & \pmb{0.065} \tabularnewline
\specialrule{.15em}{.07em}{.07em} 
\end{tabular}\caption{\small{Result for segmentation map generation. We run our method for 5 independent runs and report the averaged performance. \ours gives 0.003/0.006 standard deviation of ELBO/IWBO.}} \label{tbl:cate}
\par
\vspace{-0.2cm}
\end{centering}
\end{wraptable}

\paragraph{Segmentation Map Generation} \label{sec:exp_seg}
\ours can also be applied to generate high dimensional categorical distribution such as the segmentation map of a 2D image. Following \citet{hoogeboom2021argmax}, we aim to learn a model to generate the segmentation map of cityscapes dataset, in which the value of each pixel represents the category of the object that pixel belongs to. Following the setup in \citet{hoogeboom2021argmax}, 
there are in total 8 categories and the value at each pixel is coded using one-hot vector. We compare our approach with uniform dequantization \citep{uria2013rnade}, variational dequantization \citep{ho2019flow++}, three variants of argmax flow \citep{hoogeboom2021argmax} and multinomial diffusion \citep{hoogeboom2021argmax}.
Following \citet{hoogeboom2021argmax}, we evaluate the quality of generative model by evidence lower bound (ELBO) and importance weighted bound (IWBO) \citep{burda2015importance} (when it is available) with 1000 samples measured in bits per pixel. For our method, we run 5 independent trials and report the averaged metric and its standard deviation. The other baselines are directly adopted from \citet{hoogeboom2021argmax}. The result is summarized in Table \ref{tbl:cate}. See Appendix \ref{apx:exp_det} for additional details.

\subsection{Analysis} \label{sec:analysis}
\paragraph{Hitting time distribution}

We study the hitting time distribution given by the optimized network, which is summarized in Figure \ref{fig:hit_hist}. Our first hitting diffusion model is able to hit the domain in a well-bounded time. It is worth remarking that for Boolean and categorical distribution, \emph{\ours generates higher quality samples with much fewer diffusion steps.} For example, in graph generation, \ours on average takes about 100 steps while the previous approach such as \citet{niu2020permutation} requires 6K steps. Similarly, in segmentation map generation, \ours takes about 90 steps on average while the multinomial diffusion \citep{hoogeboom2021argmax} needs 4K steps. Decreasing the number of diffusion steps in those approaches will degenerate the performance. For example, if we only use 120 diffusion steps in \citet{niu2020permutation} the averaged performance becomes 0.306 which is much worse. See Appendix \ref{apx:exp_res} for detailed result.

\begin{wrapfigure}{r}{0.3\textwidth}
\vspace{-0.7cm}
    \centering
    \includegraphics[width=0.3\textwidth]{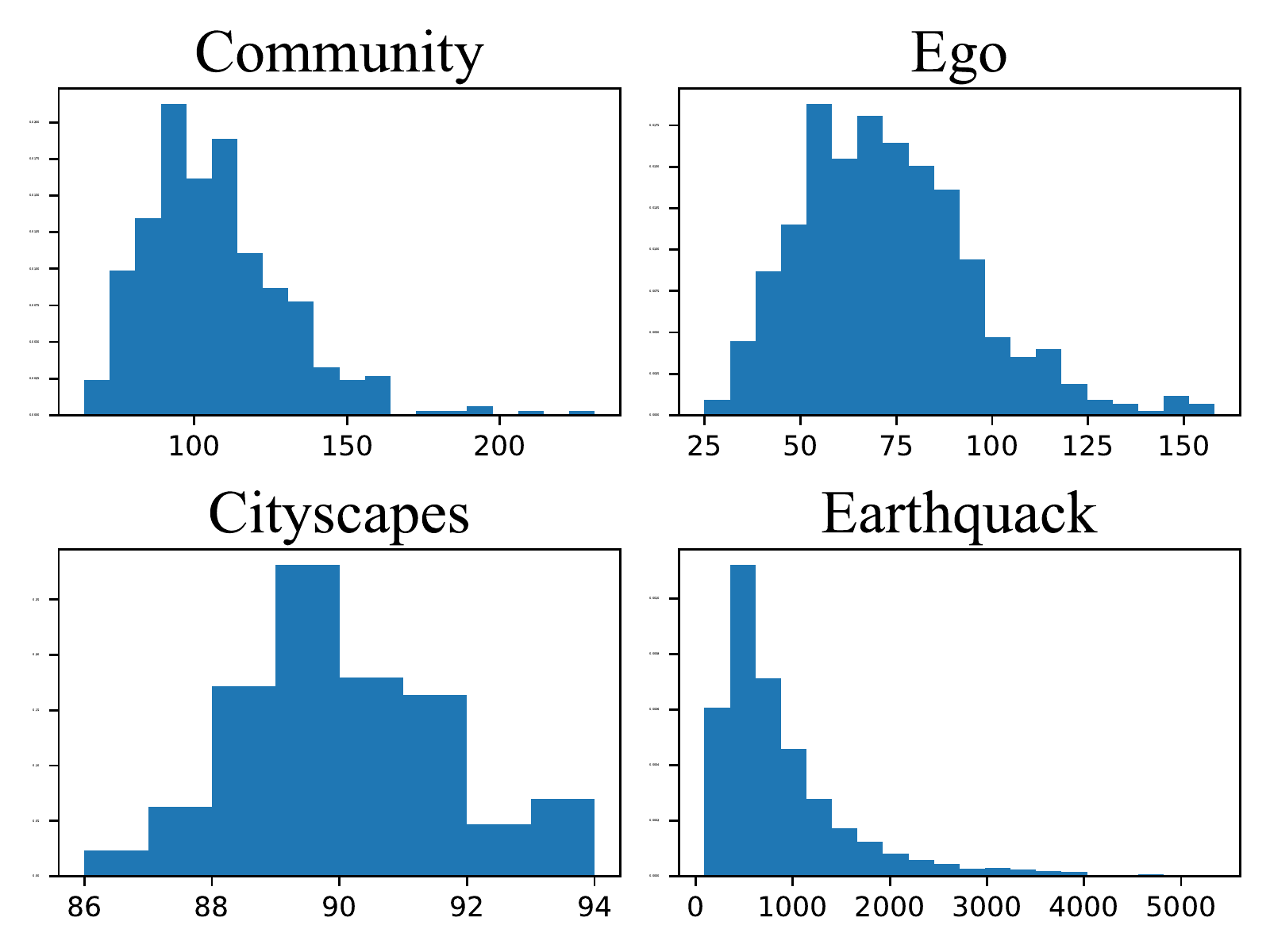}
    \vspace{-0.7cm}
    \caption{{Hitting time distributions for different data distributions.}} 
    \vspace{-0.3cm}
    \label{fig:hit_hist}
\end{wrapfigure}

\paragraph{Why we can stop at hitting time}
The key feature of \ours that is we stop the diffusion when it hits the domain rather than keep it running for a pre-fixed time. We explore more the intuition behind such a process. In figure \ref{fig:hit_traj} we visualize the trajectory of generating a segmentation map. By looking at the image snapshot in the upper row where the value of each pixel is decided by the argmax (i.e. rounding) of the 8 scores, we observe that the global contour of the image is already determined at a very early time (i.e., step 20) while the refinement of local details is almost finished at step 50. Our first hitting model exploits such property to stop the diffusion of the hit pixels that the model has enough confidence about its value making the generating process of the rest pixels easier.

\section{Conclusion}

We propose the first hitting diffusion model (FHDM), which generalizes the fixed-time diffusion process and allows instance-dependent adaptive diffusion steps. Leveraging the idea of exit distribution, FHDM provides an unified framework for learning distribution in various special domains. Despite the good functionality, FHDM takes slightly larger training overhead, which is partially solved by the our fast sampling tricks.

\clearpage
\bibliographystyle{plainnat}
\bibliography{diffusion_generative}

\newpage \clearpage 
\appendix
\section{Appendix} 

\subsection{Algorithm for Learning Categorical Data} \label{apx:cate}
The over all training algorithm for learning categorical generative models is similar to the other cases. To simulate the conditioned process, given any exit point $x\in C_{d,m}\subseteq B_{d}$, we know that $\mathbb{Q}^{C_{d,m}}(\cdot\mid Z_{\tau}=x)=\mathbb{Q}^{B_d}(\cdot\mid Z_{\tau}=x)$ and thus (\ref{equ:Xtrsphere}) can be reused. The training of the network is also similar, the only difference is that we have the additional term $\nabla_{z}\log\text{Ber}(\Omega\mid Z_{t})$ in the output of the network to ensures that the generative process in proposition \ref{thm:feq} is guaranteed to hit $\Omega$ when it exists $V$. The training loss is thus
\begin{align*}
\mathcal{L}(\theta) & =\frac{1}{2}\mathbb{E}_{\substack{x\sim\Pi^{*} \\ Z\sim\mathbb{Q}^{x}}}\Big[\int_{0}^{\tau}\left\Vert \mathbb{I}\{Z_{t}\in C_{d,m}\}\circ\left(f_{t}^{\theta}(Z_{t})+\nabla_{Z_{t}}\log\text{Ber}(\Omega\mid Z)-\nabla_{Z_{t}}\log\text{Ber}(Z_{t}\mid x)\right)\right\Vert ^{2}dt
\\
& -\log p_{0}^{\theta}(Z_{0})\Big]+const.
\end{align*}

\subsection{Practical Algorithm}
We give a detailed practical algorithm. 

\textbf{Discretized process}
Suppose that the diffusion step size at step $k$ is $\epsilon_k$. Given the exit point $x$, the discretized conditioned process can be simulated by 
\bbb \label{equ:bridgeXh_dis}
\X_{t_{k+1}} = \left (b_{t_{k}}(\X_{t_{k}}) + \blue{\sigma^2_{t_{k}}(\X_{t_k}) \dd_{z} \log h^\tg_{t_{k}}(\X_{t_{k}}) } \right ) \epsilon_k + \sqrt{\epsilon_k}\sigma_{t_{k}}(\X_{t_{k}}) \xi_k, ~~~ 
\X_{t_0} \sim \Q^\tg_0,
\eee
where $\xi_k \sim \mathcal{N}(0,I)$ is a standard Gaussian noise. Note that (\ref{equ:bridgeXh_dis}) is terminated at $t_k$ when $\X_{t_k}$ firstly hits the desired domain.
Alternatively, we can first sample (discretized) unconditioned process by 
\bbb \label{equ:uncond_bridge_dis}
\X_{t_{k+1}} = b_{t_{k}}(\X_{t_{k}}) \epsilon_k + \sqrt{\epsilon_k}\sigma_{t_{k}}(\X_{t_{k}}) \xi_k, ~~~ 
\X_{t_0} \sim \Q^\tg_0,
\eee
And then apply the rotation operators defined in Section \ref{sec:fast} such that the sampled trajectory ends at $x$.

\textbf{A simplified loss}
Similar to \citet{song2019generative,ho2020denoising}, we use a stochastic version of loss (\ref{equ:mainloss}), in which we only uniformly sample temporal snapshots to compute the loss.
 \bbb \label{equ:mainloss_sto} 
  \hat{\L}(\theta) 
 & =  \frac{1}{2} \E_{\Q^\tg} \E_{t\sim\text{Unif}\{0,...,\tau\}}
 \!\!\!\left [ \norm{
 \sigma_t(\X_t)^{-1}(s^\theta_t(\X_t) -
 b_t(\X_t~|~ \X_\tau))
 }^2 - \log p_0^\theta(\X_0) \right ] + \const.
\eee

In Algorithm \ref{alg:learningmain_prac}, we summarize the training procedure of FHDM.

\begin{algorithm}[h] 
\label{alg:learningmain_prac}
\caption{Learning Generative Models by First Hitting Diffusion} 
\begin{algorithmic}
\STATE \textbf{Inputs}: A data $\{x\datai\}$ drawn from $\tg$ on $\Omega$.
A baseline process $\Q$ and a model $\P^\theta$ that are absorbing to $\Omega$. 
\STATE \textbf{Goal}: Find $\theta$ such that  $\P_\Omega^\theta \approx \tg$. 
\STATE \textbf{Training}: 
By minimizing $\L(\theta)$.

\STATE{(Optional) Pre-simulate unconditioned trajectories of $\mathbb{Q}$ using (\ref{equ:uncond_bridge_dis}).}
\FOR{training iters}
    \STATE{Get a mini batch of data from training set.}
    \STATE{//Optionally, we can use fast bridge sampling tricks to get conditioned sample by rotating ~// pre-simulated unconditioned trajectories.}
    \STATE{Sample trajectories $\mathbb{Q}(\cdot \mid Z_\tau = x)$ for each data $x$ in the mini batch using (\ref{equ:bridgeXh_dis})}
    \STATE{Calculate the mini-batch loss $\mathcal{L}(\theta)$ defined in Equ~(\ref{equ:mainloss_sto}).}
    \STATE{Apply gradient descent to update $\theta$.}
\ENDFOR

\end{algorithmic}
\end{algorithm} 

\subsection{Sampling with first hitting $h$-transform} \label{apx:sample}

The $h$-transform formula on first hitting diffusion readily provides a simple mechanism for approximate sampling from $\tg$: Assume the baseline process $X$ is designed simple enough such that the conditional harmonic measure $\Q_{\Omega}(\cdot ~|~ Z_t=z)$ is easy to calculate, then we can approximately $h_t^{\tg}(z)$ in \eqref{equ:barh} 
 by Monte Carlo sampling from  $\Q_{\Omega}(\cdot ~|~ Z_t=z)$:
 $$
 h_t^\tg(z) \approx \frac{1}{m}\sum_{i=1}^m \pi^*(x\datai),\ \ 
 x\datai \sim \Q_\Omega(\cdot ~|~ Z_t = z),
 $$ 
 use it simulate process \eqref{equ:bridgeXh}. 
 The gradient $\dd \log h^\tg_t$ can be approximated with either the reparameterization method or score function method. 
 See Algorithm~\ref{alg:sampling}. 
\begin{algorithm}[h]
\label{alg:sampling}
\caption{Approximate Sampling by First Hitting Diffusion}
\begin{algorithmic}
\STATE \textbf{Goal}: Draw sample from $\tg$ on $\Omega \in \RR^d$. 
\STATE \textbf{Prepare} a baseline diffusion process $\X \sim  \itoo(b,\sigma)$ in  \eqref{equ:baseX} with exit distribution  $\Q_\Omega (A ~|~ \Z_t = z) = \Q(\X_\tau\in A ~|~\Z_t=z)$. %
 Let $h = {\df \tg}/{\df \Q_{\Omega}(\cdot| \Z_0= z_0)}$ be the density ratio between $\tg$ and $\Q_{\Omega}(\cdot| Z_0= z_0)$, 
 where the initialization $Z_0 = z_0$ is in $V\setminus \Omega$. %
\STATE \textbf{Simulate} the following process $\{\hat Z_t\}$ starting from $\hat Z_0= z_0$ and  stop at the first hitting time $\tau = \inf \{t\geq 0 \colon \hat Z_t\in\Omega\}$: 
\bbb \label{equ:bridgeXhat}
\df \hat \X_t = \left (b_t(\hat \X_t) + \sigma_t^2(\hat \X_t) \dd_{z} \log\hat h_t(\X_t)  \right ) \dt + \sigma_t(\hat \X_t) \df W_t,
\eee
where $\hat h_t(z) = \frac{1}{m}  \sum_{i = 1}^m \tgd(x\datai),$ 
where $\{x\datai\}_{i=1}^m$ is drawn i.i.d. from $\Q_\Omega (\cdot ~|~ \Z_t = z)$; the derivative $\dd_{z} \log\hat h_t (z)$ can be calculated by either the reparameterization trick or score function method.  
\STATE \textbf{Return} $\hat \X_\tau$ as an approximate draw from $\tg$. 
\end{algorithmic}
\end{algorithm} 

\subsection{Connection with SMLD and DDPM} \label{apx:compare}
Standard diffusion generative models such as SMLD and DDPM determinates the diffusion process at a fixed time, which can be included as a special first hitting model as shown in Example~\ref{exa:fixedtime}. 
We clarify the connection to SMLD and DDPM for completeness here. 
In this case, we set $\Q$ to be an Ornstein-Uhlenbeck (O-U) process $\d Z_t = \alpha_t \Z_t \dt + \sigma_t \d W_t$ initialized at $Z_0 \sim \normal(\mu_0, v_0)$ and stopped at a deterministic time $t = \t$,  where $\alpha_t\in \RR$ and  $\sigma_t\geq 0$, $v_0\geq 0$, $\forall t$.
This is a Gaussian process. Let $Z_t \sim\normal(\mu_t, v_t)$. 
Denote by $\bar Z_t = Z_{\t-t}$ the time reversed process, which follows \citep{anderson1982reverse}
\bb 
\d \bar Z_t 
= \left (- \alpha_{\t-t} \bar Z_t + 
\sigma_{\t-t}^2  
\frac{\mu_{\t-t} - \bar Z_t}{v_{\t-t}}  \right ) \dt + \sigma_{\t-t} \d \bar W_t,
\ee 
where $\bar W_t$ is a copy of standard Brownian motion. 
If we set $v_0 \to +\infty$ in the initial $Z_0$, we expect to have $v_t \to +\infty$ under proper regularity conditions on $\alpha_t$ and $\sigma_t$, the second term in the drift of $\bar Z_t$ is canceled, yielding $\d \bar Z_t = - \alpha_{\t -t} \bar Z_t \dt + \sigma_{\t - t} \d \bar W_t$. 
This then reduces to the  processes used in SMLD ($\alpha_t = 0$), 
and DDPM and SDE method in \cite{song2020score}  ($\alpha_t>0$).
This framework of learning fixed-time diffusion models using bridge processes 
are explored separately in a recent work \cite{peluchetti2021non}. 
The authors devote more in-depth discussions on the fixed-time diffusion case 
in a separate work.  

\subsection{Additional Experiment Details} \label{apx:exp_det}
\subsubsection{Point Cloud Generation}
\textbf{Training details}
The ShapeNet dataset contains 51,127 shapes from 55 categories and is randomly split training, testing and validation set by the ratio 80\%, 15\% and 5\%. For each shape, we sample 2048 points to acquire the point clouds and normalize each of them to zero mean and unit variance.

We build our method on \citet{luo2021diffusion} in which the encoder of a flow-based model is used to learn a latent code of the shape and conditioning on the shape latent code, the point are independently generated based on a diffusion model. We substitute the DDPM-type \citep{ho2020denoising} of diffusion model with ours and all the other components remain the same. Each point is generated using 100 diffusion steps and the step size linearly decays starting from $0.02$ to $10^{-4}$. We use the same network architecture for flow-based model and point diffusion network. We train the model for 1M steps with batch size 128 using Adam optimizer \citep{kingma2014adam}.

\textbf{More Details on Evaluation Metrics}
Both MMD, JSD and 1-NN measures the fidelity of the generated samples. The 1-NN score is the accuracy of 1-NN classifier in predicting whether a point cloud is generated by the model or from the data. Lower 1-NN scores suggest higher quality. MMD and JSD measures the probability distance between the point distributions of the generated set and the reference set from data and thus lower MMD and JSD means higher quality. COV detects mode-collapse and higher COV suggests more diverse generated samples. 

\subsubsection{Generating Distribution on Sphere}
\textbf{Training details}
All datasets are split into training, validation and test sets with $(0.8, 0.1, 0.1)$ proportions. We train the model for 2000 iterations using Adam Optimizer \citep{kingma2014adam} with learning rate $0.05$ and batch size 128. We use a three-layer MLP with 100 hidden units and ReLU activation to approximate the drift. We set the maximum diffusion step as 10K with step size $5\times 10^{-4}$. The model on average takes 1K steps to hit and seldom takes more than 5K to hit. See section \ref{sec:analysis} for more details on the hitting time distribution.

\subsubsection{Graph Generation}
\textbf{Training details}
We set the maximum number of SDE steps as 10K, and it takes on average about 100 steps to hit. See section \ref{sec:analysis} for more detailed analysis. At each step, we set the standard deviation of gaussian noise as 0.5. We initialize all the coordinate 0.5 and stop the updating of a coordinate at the first time its distance to 0 or 1 is less than 0.05. We use the same network architecture and training pipeline as \citet{niu2020permutation}. Adam optimizer \citep{kingma2014adam} with 0.001 learning rate is applied. Batch size is set to 32 and for each graph, we randomly sample 6 snapshot in the trajectory for training. The score matching loss of a hit coordinate is masked out at training.

\subsubsection{Segmentation Map Generation}
\textbf{Training details}
Our network architecture and training pipeline is almost the same as the multinomial diffusion model proposed in \citet{hoogeboom2021argmax}. The only architecture difference is that \citet{hoogeboom2021argmax} first feed the image into an embedding layer before passing to the subquential U-Net \citep{ronneberger2015u} like structure while we use a linear layer with the same output dimension. This is because the multinomial diffusion model \citep{hoogeboom2021argmax} is a discrete diffusion in which the value at each pixel is considered to be discrete while \ours is a continuous diffusion. We set the number of maximum diffusion steps to be 100 and the step size to be 0.1. We apply step-decayed Gaussian noise at different diffusion steps, in which the standard deviation at initial is 1 and decay to half at step 500 and 750. The pixel is hit and stopped to update at the first time its largest categorical score (among 8 of them) is greater than $1-\epsilon$ with $\epsilon=0.01$. We apply the same data augmentation and train the model for 500 epochs with batch size 64, learning rate $10^{-4}$ and Adam optimizer \citep{kingma2014adam}. For each image in the batch, we randomly sample one time snapshot along the diffusion trajectory for training.

For this task, we apply the fast bridge sampling method proposed in Section \ref{sec:fast}. At the beginning of each epoch, we generate $10 \times \text{batch size} \times H \times W$ unconditional SDE trajectories where $H,W$ is the height and width of the images. At the training time, to simulate the SDE trajectories of a given image in the training set, for each pixel, we randomly select one saved unconditional SDE trajectories and rotates it such that it ends at that pixel.

\subsection{Additional Experiment Results} \label{apx:exp_res}
\textbf{Number of diffusion steps}
When we restrict the number of diffusion steps of EDP-GNN \citep{niu2020permutation}, the second-best approach, similar to that of FHDM (120), we observe a significant performance drop. As shown in Table \ref{tbl:steps}, the performance of EDP-GNN degenerates badly when we decrease its diffusion steps from 4K to 120.

\begin{table}
\centering{}%
\resizebox{\columnwidth}{!}{
\begin{tabular}{c|ccccccccc}
\specialrule{.15em}{.07em}{.07em} 
\multirow{2}{*}{Method} & \multicolumn{4}{c}{Community-small} & \multicolumn{4}{c}{Ego-small} & \multirow{2}{*}{Avg}\tabularnewline
\cline{2-9} \cline{3-9} \cline{4-9} \cline{5-9} \cline{6-9} \cline{7-9} \cline{8-9} \cline{9-9} 
 & Deg. & Clus. & Orbit. & Avg. & Deg. & Clus. & Orbit. & Avg. & \tabularnewline
\hline 
EDP-GNN & 0.053 & 0.144 & 0.026 & 0.074 & 0.052 & 0.093 & 0.007 & 0.050 & 0.062\tabularnewline
EDP-GNN (step=120) & 0.586 & 0.253 & 0.705 & 0.515 & 0.141 & 0.114 & 0.036 & 0.097 & 0.306\tabularnewline
\rowcolor{lightgray}Ours & \pmb{0.004} & \pmb{0.104} & \pmb{0.001} & \pmb{0.036} & \pmb{0.019} & \pmb{0.047} & \pmb{0.005} & \pmb{0.024} & \pmb{0.030}\tabularnewline
\specialrule{.15em}{.07em}{.07em} 
\end{tabular}
}
\caption{Comparing FHDM with EDP-GNN with similar diffusion steps.} \label{tbl:steps}
\end{table}

\textbf{Acceleration by fast sampling}
We give brief analysis on the acceleration effect of the fast bridge sampling method described in Section \ref{sec:fast}. When applied to the segmentation generation experiment, we pre-simulate 640 trajectories in the beginning of each training epoch which gives 2.5x acceleration from 24.3 cpu time/epoch to 9.8 cpu time/epoch, making the training time of FHDM is comparable to \citet{hoogeboom2021argmax} (6.5 cpu time/epoch). We remark that although FHDM has slightly larger training overhead, its only requires less than 100 diffusion steps at inference, giving a 40x speed up compared with \citet{hoogeboom2021argmax}.

\textbf{Ablation studies on noise schedule}
\begin{figure*}[t]
    \centering
    \includegraphics[width=\textwidth]{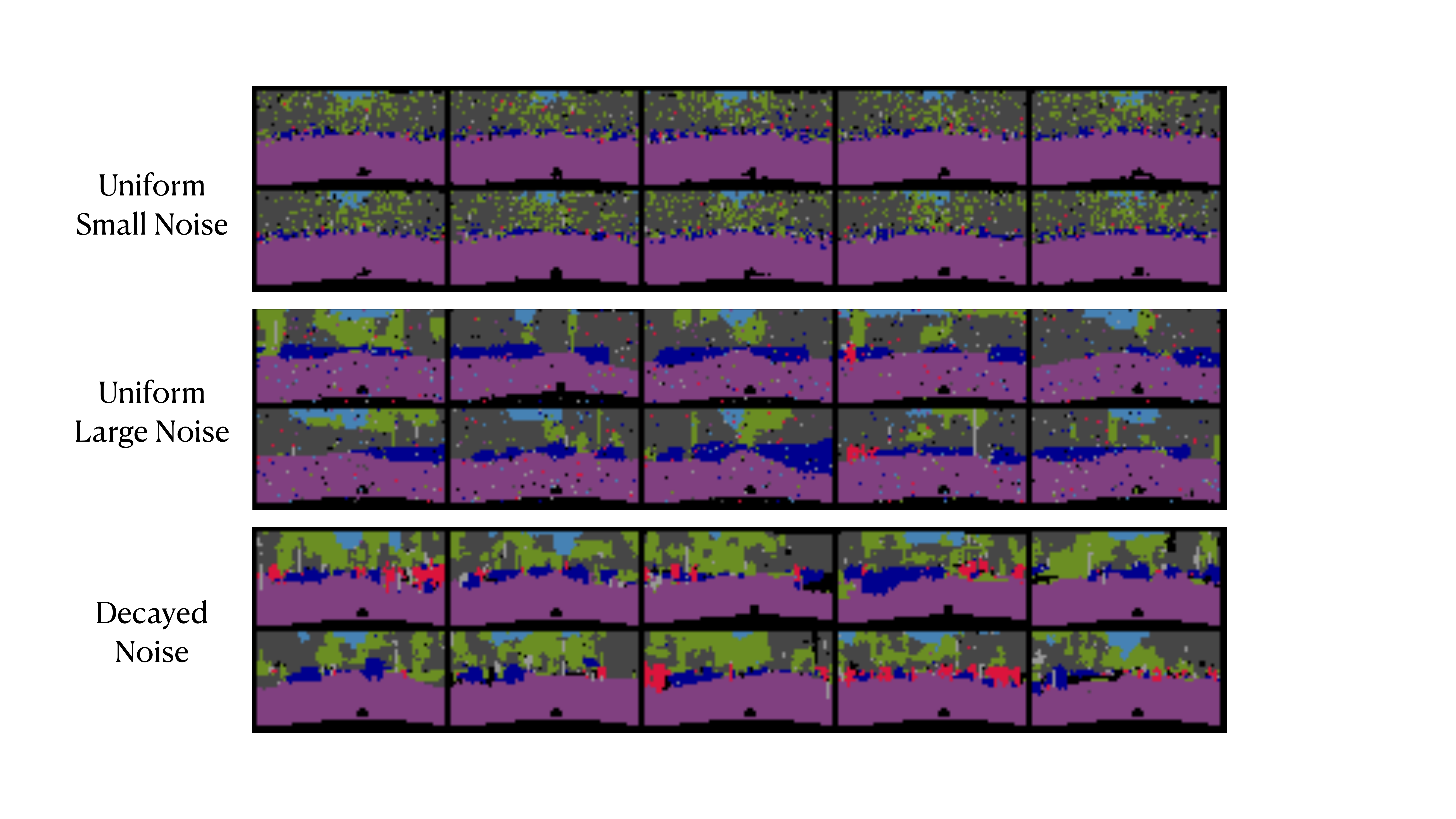}
    \caption{Compare the generated segmentation maps with different noise schedule.}
    \label{fig:compare_noise}
\end{figure*}
In practice, we observe that the design of the noise schedule can be important for some tasks such as segmentation generation. We show samples generated by FHDM with uniformly small noise (std=0.25), uniformly large noise (std=1) and decayed noise as described in Section in \ref{sec:exp_seg}. It is worth noticing that using a uniformly small noise generates over-smoothed and degenerated images that fail to reveal the details while using a uniformly large noise gives more diverse but noisy images. In comparison, the decaying noise generates high-quality diverse images with fine details.

\subsection{Visualization of Generated Samples} \label{apx:vis}

\textbf{Point Cloud Generation}
Please see Figure~\ref{fig:air_all} and \ref{fig:chair_all} for the generated airplane and chair point cloud using FHDM.

\begin{figure*}[t]
    \centering
    \includegraphics[width=\textwidth]{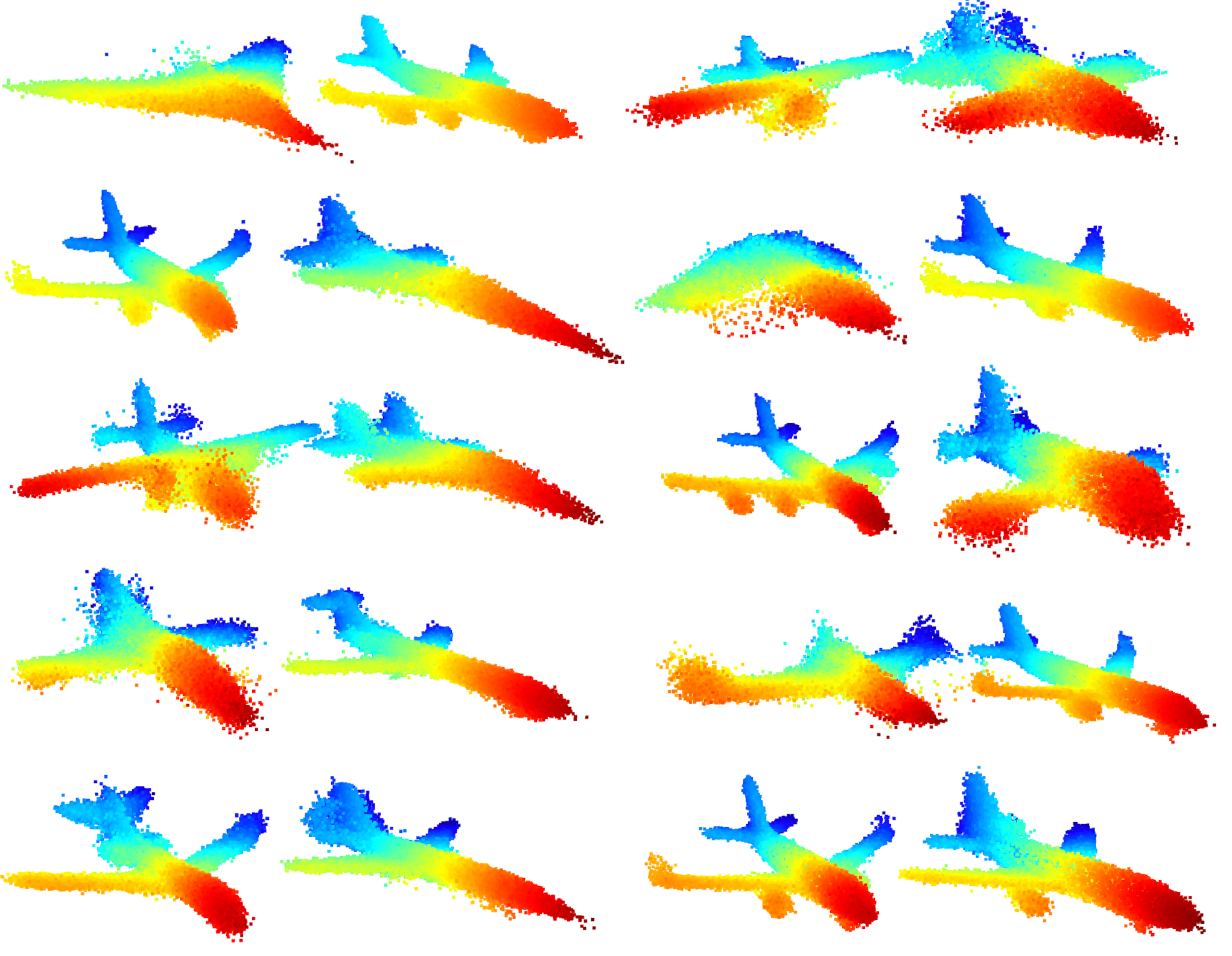}
    \caption{The generated airplane point cloud by FHDM.}
    \label{fig:air_all}
\end{figure*}

\begin{figure*}[t]
    \centering
    \includegraphics[width=\textwidth]{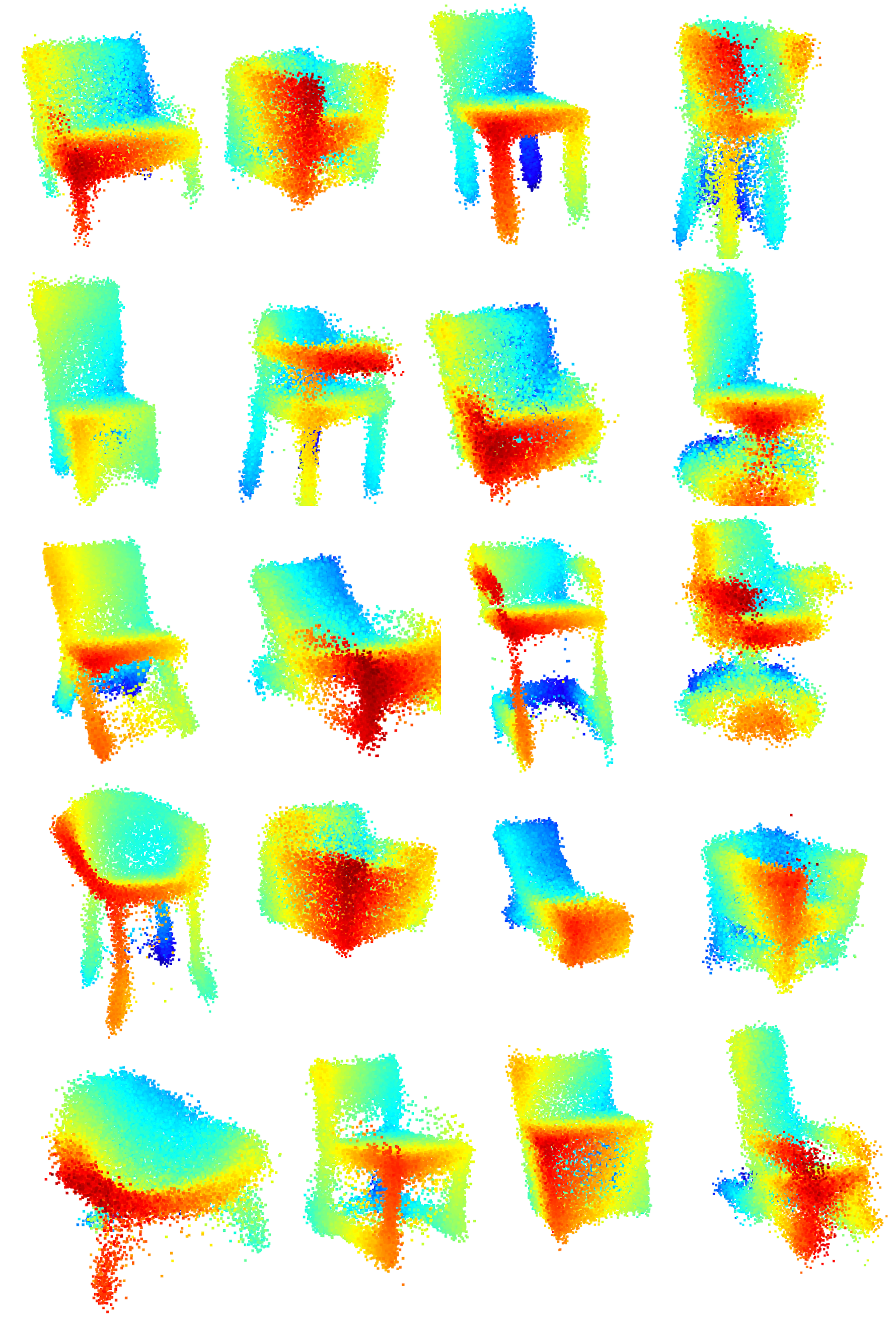}
    \caption{The generated chair point cloud by FHDM.}
    \label{fig:chair_all}
\end{figure*}

\textbf{Generating Distribution on Sphere}
Please see Figure~\ref{fig:ego} for the generated graphs using FHDM.

\textbf{Segmentation Map Generation}
Please see Figure~\ref{fig:sphere_all} for the generated distributions on sphere by FHDM.

\begin{figure*}[t]
    \centering
    \includegraphics[width=\textwidth]{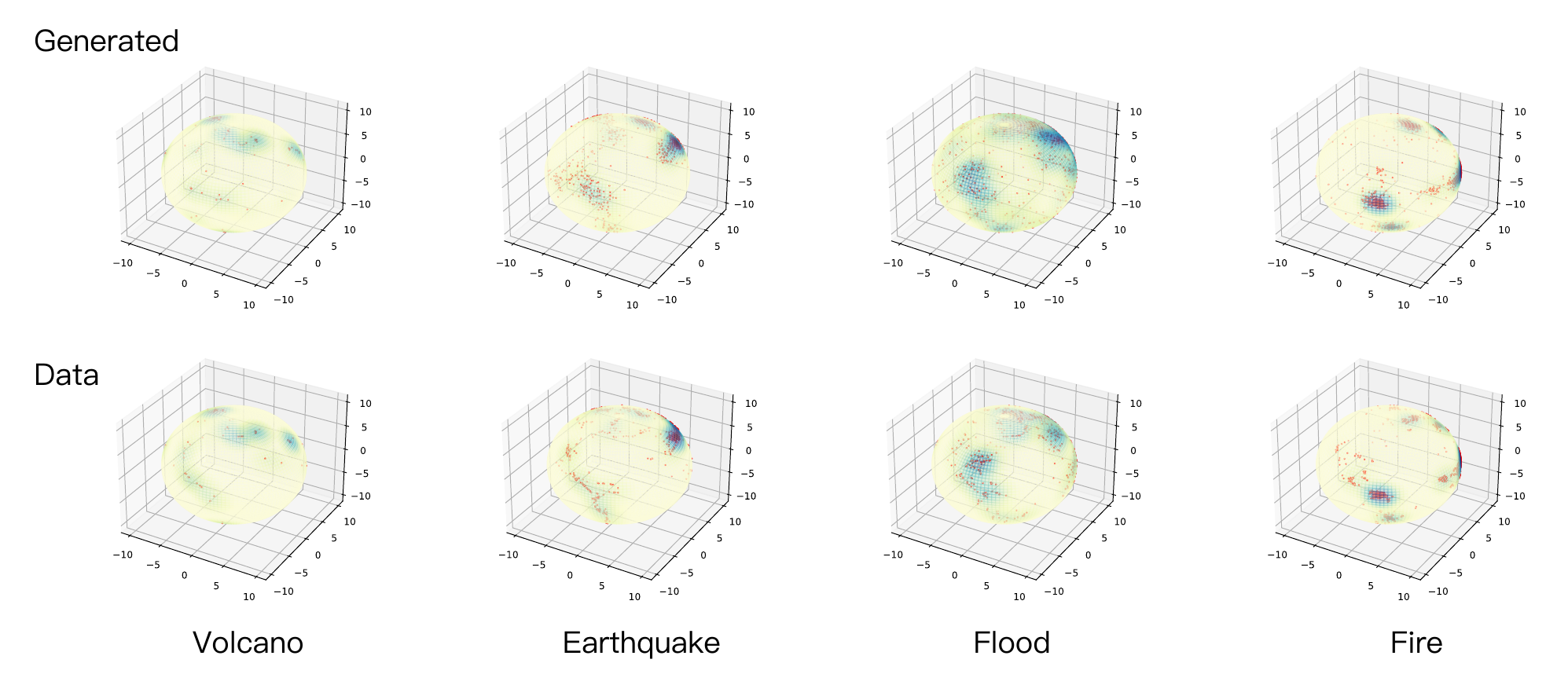}
    \caption{The generated distrubution on sphere by FHDM.}
    \label{fig:sphere_all}
\end{figure*}

\textbf{Graph Generation}
Please see Figure~\ref{fig:ego} for the generated graphs using FHDM.

\textbf{Segmentation Map Generation}
Please see Figure~\ref{fig:seg_all} for the generated segmentation maps by FHDM.

\begin{figure*}[t]
    \centering
    \includegraphics[width=\textwidth]{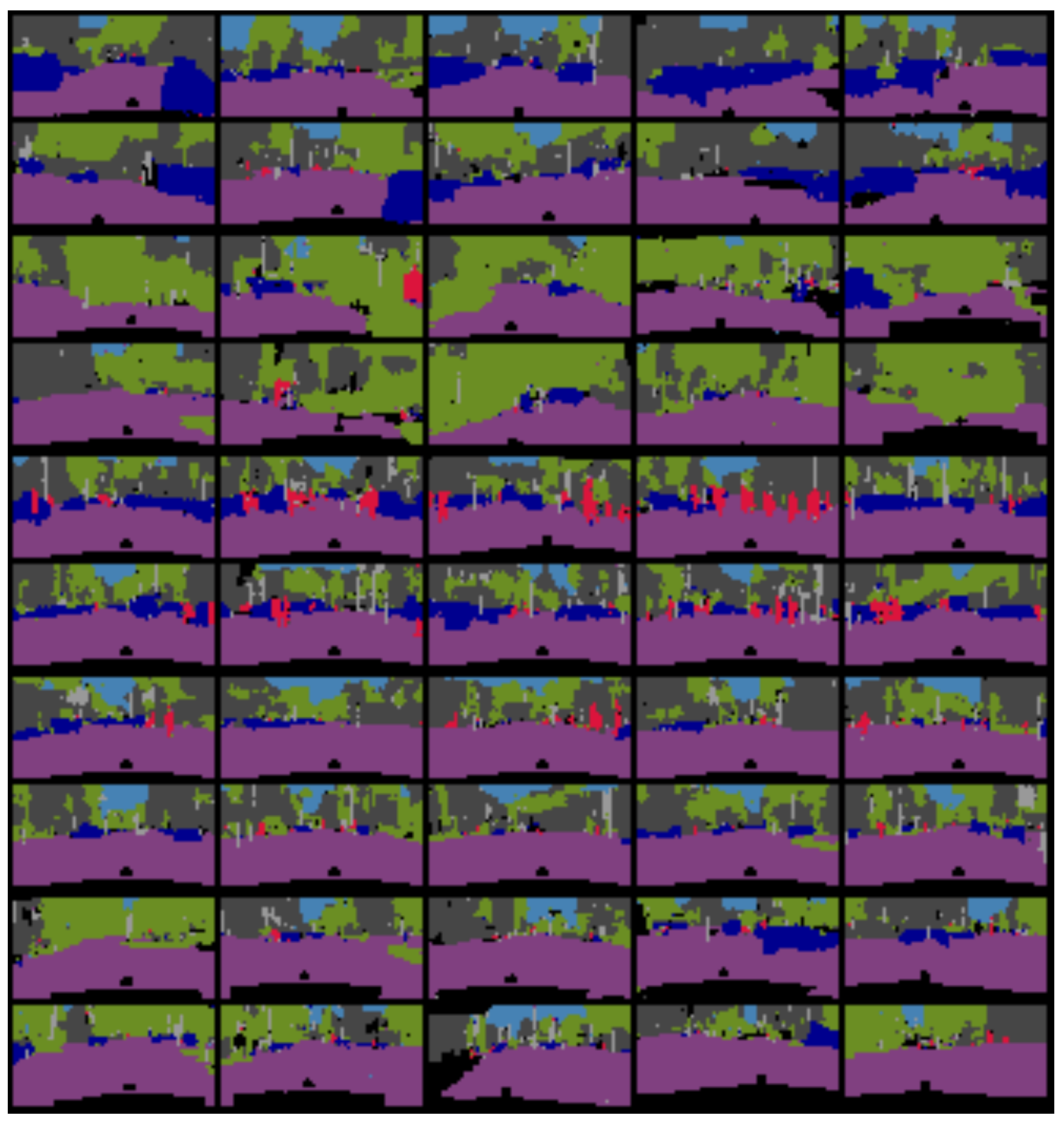}
    \caption{The generated segmentation maps by FHDM.}
    \label{fig:seg_all}
\end{figure*}

\begin{figure*}[t]
    \centering
    \includegraphics[width=0.49\textwidth]{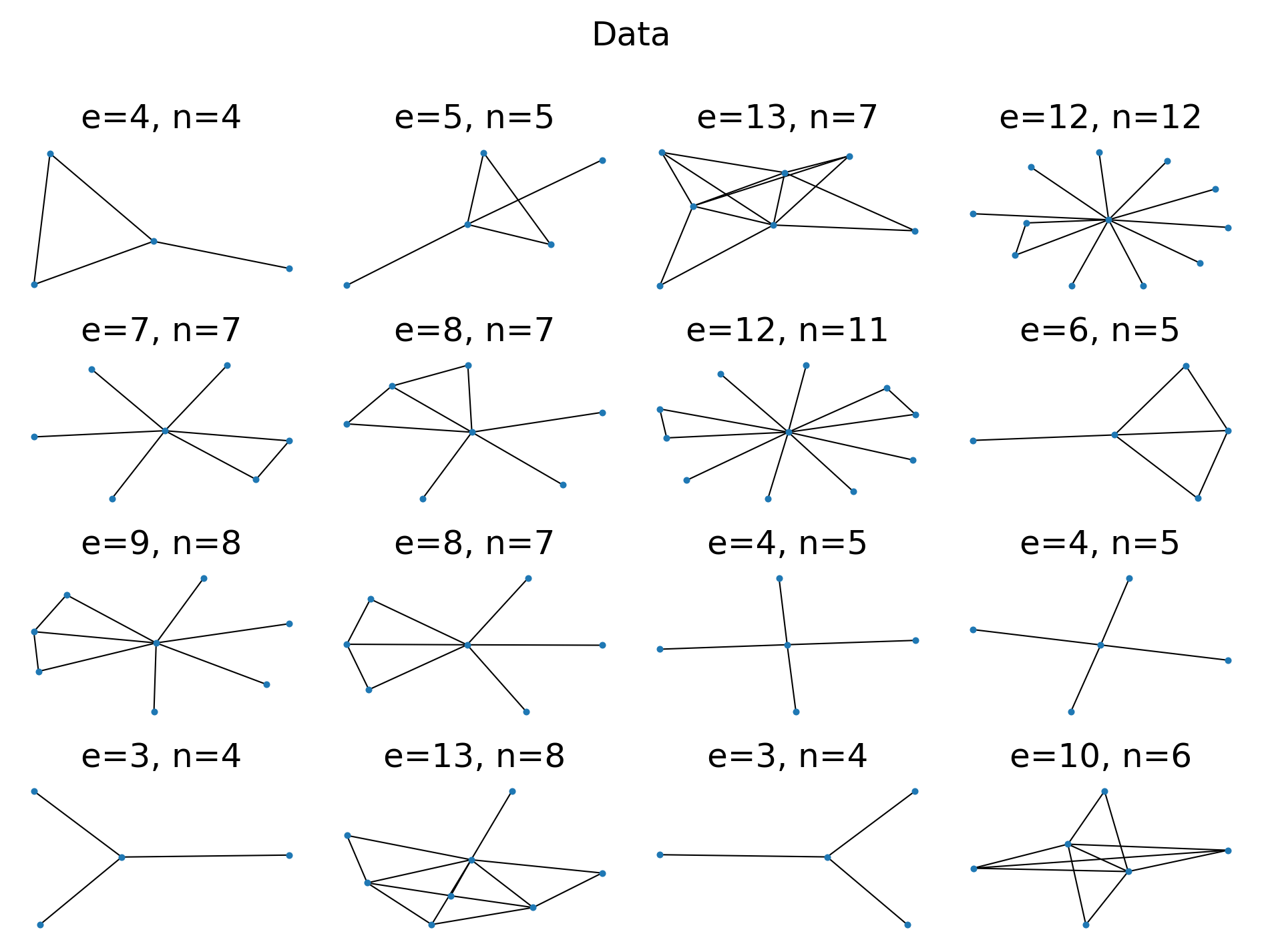}
    \includegraphics[width=0.49\textwidth]{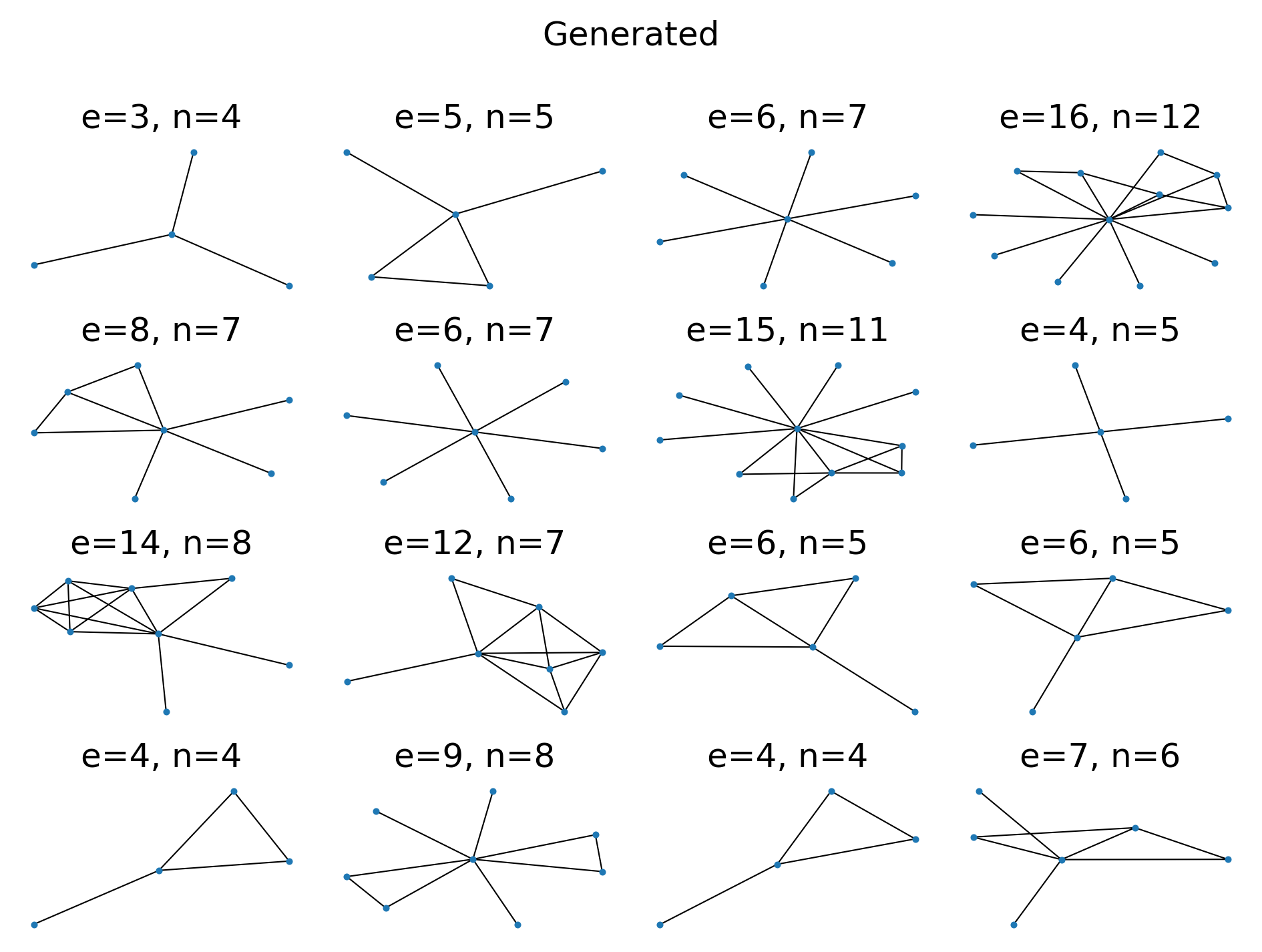}
    \includegraphics[width=0.49\textwidth]{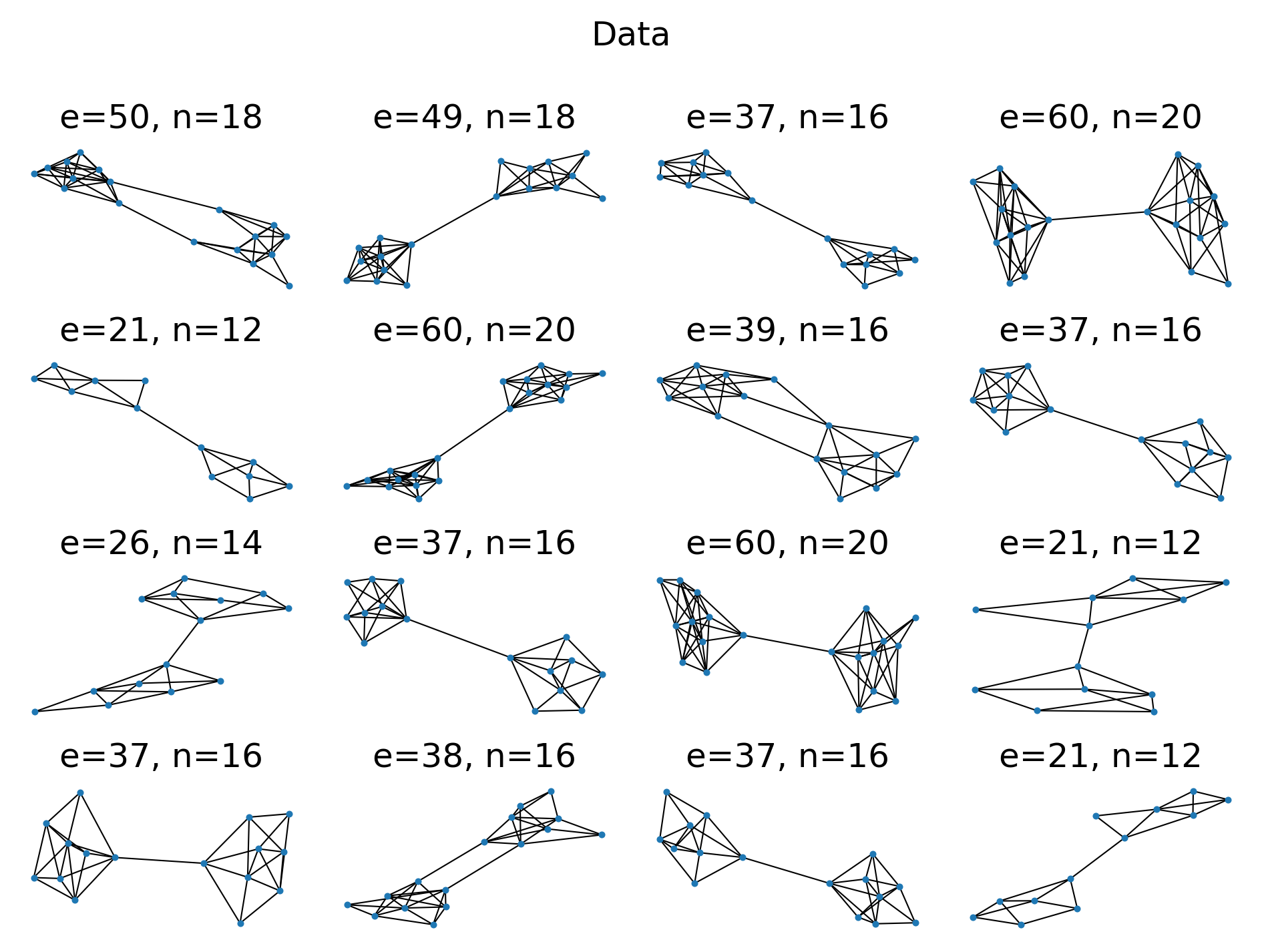}
    \includegraphics[width=0.49\textwidth]{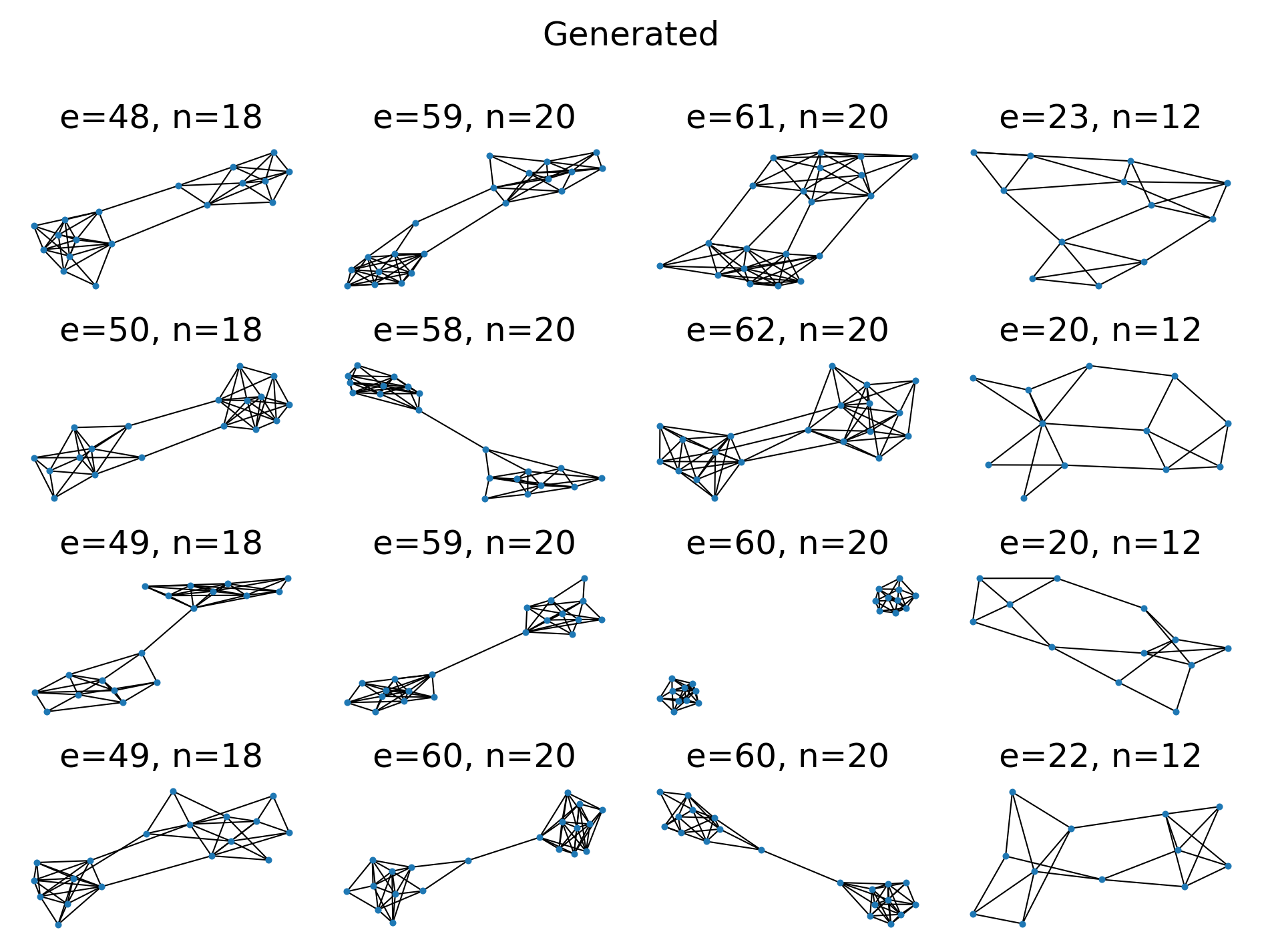}
    \caption{True and generated graphs of ego (upper rows) and community (lower rows) datasets.}
    \label{fig:ego}
\end{figure*}

\subsection{Discretization Error Analysis} \label{apx:discret}
Consider the following Ito process
\[
dZ_{t}=b(Z_{t})dt+\sigma(Z_{t})dW_{t}
\]
and a open subset $V\subseteq\R^{d}$. Here $b\in \mathbb{R}^d$ is the drift and $\sigma \in \mathbb{R}^{d\times d}$ is the diffusion matrix. We stop the process when $Z_{t}$
hit the domain $\Omega$ and denote the hitting time as $\tau:=\inf_{t\ge0}\{Z_{t} \notin V\}$.
We consider the discretalization error of the conditional distribution
$\pi_{T}$ with temporal truncation $T$, i.e, 
\[
\pi_{T}:=\text{law of}\ X_{\tau}\mid\tau\le T.
\] This corresponds to the situation that we discard the non-hit process after waiting for $T$ time. To simulate the above process, we consider the Euler discretalization
on $[0,T]$. Suppose $R$ is a set of grid points on $[0,T]$
in which we define 
\[
r_{t}=\max\{r\in R:r\le t\}.
\]
The Euler discretized process is thus defined as 
\[
d\bar{Z}_{t}=b(\bar{Z}_{r_{t}})+\sigma(\bar{Z}_{r_{t}})dW_{t}.
\]
And similarily, we can define its (discretize) stopping time $\bt\defeq \min_{r\in R}\{r \colon ~\bar{Z}_{r_{t}}\notin V\}$.
We want to bound the discrepancy between $\pi_{T}$  and the following distribution 
\[
\bar{\pi}_{T}:=\text{law of }\bx_{\bt}\mid\bt\le T.
\]
We consider the Wassestein distance $\W[\bar{\pi}_{T},\pi_{T}]$
for measuring the discrepancy. In this section, $||\cdot||$ is vector norm when applied to vector and is matrix operator norm when applied to matrix.

\begin{assumption} \label{asm:reg} 
$b$ and $\sigma$ is $L$-Lipschitz and $\sup_{z}(||b(z)||+||\sigma(z)||)\le L$.
\end{assumption}

\begin{assumption} \label{asm:char} 
There exists a bounded $C_{b}^{2}$ function $\delta:\mathbb{R}^{d}\to\mathbb{R}$
such that $\delta>0$ on $V$, $\delta=0$ on $\Omega$
and $\delta<0$ on $\mathbb{R}^{d}\setminus (V\cup\Omega)$ and satisfies the
non-characteristic boundary condition $\norm{\sigma \nabla\delta}\ge2L^{-1}$ on $\{\norm{\delta}\le r\}$ for some $r>0$.
\end{assumption}

Assumption \ref{asm:reg} is a standard assumption on the Lipschitz continuity and boundedness on the drift and diffusion function. Assumption \ref{asm:char} is more on a technical condition and is introduced in \citet{bouchard2017first} and intuitively it can be understood in the way that there exists a bounded smooth function that can indicate whether we are within $V$ or out of $V$.

\begin{theorem}
Let $\Delta:=\min_{r_{t}\neq r_{t'}}|r_{t}-r_{t'}|$. Under Assumption
\ref{asm:reg}, \ref{asm:char} and assume that $T$ is properly large such that $\P(\tau\ge T-1)\le1/4$,
we have, there exists $\epsilon>0$ such that for any $\Delta\le\epsilon$,
\[
\W^{2}[\bar{\pi}_{T},\pi_{T}]=O(\exp(cT)\Delta),
\]
for some absolute constant $c<\infty$.
\end{theorem}
Intuitively, we show that when $T$ is properly large (which is true in practice as we should wait the process a reasonably enough time for hitting) and the step size is small enough, the discretalize error is small.

\begin{proof}
Throughout the proof, $c$ denotes absolute constant and may vary
in different lines. We consider the temporal augmented process $Y_{t}=[Z_{t},t]$
in which 
\[
dY_{t}=\tilde{b}(Y_{t})dt+\tilde{\sigma}(Y_{t})dW_{t}.
\]
Here $\tilde{b}$ and $\tilde{\sigma}$ are defined as 
\[
\tilde{b}(y)=\left[b(x),1\right],\tilde{\sigma}_{y}(y)\ =
\left[\begin{array}{cc}
\sigma(y) & \boldsymbol{0}\\
\boldsymbol{0}^{\top} & 0
\end{array}\right].
\]
It is not hard to verify the Lipschitz continuity and boundedness
of $\tilde{b}$ and $\tilde{\sigma}$. We also define the hitting set of the
process $Y_{t}$ by $\tilde{V}=\{y:x\in V\ \text{or}\ t < T+1\}.$
It is easy to show that $\\tilde{V}$ is a closed subset of $\R^{d+1}$
and the stopping time $\tau:=\inf_{t\ge0}\{t:Y_{t}\notin \tilde{V}\}\le T+1$.
Similarly, we can define the discretized version 
\[
d\by_{t}=\tilde{b}(\by_{r_{t}})dt+\tilde{\sigma}(\by_{r_{t}})dW_{t}.
\]
Here we slightly abuse the notation of $\tau$ and $\bar{\tau}$,
making them denoting the hitting time of process $Y_{t}$ and $\bar{Y}_{t}$
rather than $Z_{t}$ and $\bar{Z_{t}}$. We introduce the following Lemma used in \citet{bouchard2017first}.

\begin{lemma}[Theorem 3.11 in \citet{bouchard2017first}] \label{lem:exit_time}
Under assumption \ref{asm:reg} and \ref{asm:char}, there exists $\epsilon>0$ such that when
$\Delta\le\epsilon$, $\E[|\tau-\bt|]\le c\Delta^{1/2}$ for some
constant $c>0$.
\end{lemma}

Note that 
\[
\E[||\by_{\bt}-Y_{\tau}||^{2}\mid\bt\le T]\le\int\E[||\by_{\bt}-Y_{\tau}||^{2}\mid|\bt-\tau|=s,\bt\le T]\ \PP(|\bt-\tau|=s\mid\bt\le T)ds.
\]
Note that we can decompose
\begin{align*}
 & \E[||\by_{\bt}-Y_{\tau}||^{2}\mid|\bt-\tau|=s,\bt\le T]\\
\le & 2\E[||\by_{\bt}-Y_{\bt}||^{2}\mid|\bt-\tau|=s,\bt\le T]+2\E[||Y_{\bt}-Y_{\tau}||^{2}\mid|\bt-\tau|=s,\bt\le T].
\end{align*}
Using Lemma A.2 in \citet{bouchard2017first} and Holder's inequality, we have 
\[
\le\E[||\by_{\bt}-Y_{\bt}||^{2}\mid|\bt-\tau|=s,\bt\le T]\le\sup_{t\in[0,T+s]}||\by_{\bt}-Y_{\bt}||^{2}\le c\Delta,
\]
for some constant $c$. Also, by the boundedness of $\tilde{b}$
\[
||Y_{\bt}-Y_{\tau}||=||\int_{\min(\bt,\tau)}^{\max(\bt,\tau)}\tilde{b}(Y_{t})dt||\le(L+1)||\tau-\bt||.
\]
Using these two bounds, 
\[
\E[||\by_{\bt}-Y_{\tau}||^{2}\mid|\bt-\tau|=s,\bt\le T]\le c(\Delta+||\tau-\bt||^{2}).
\]
This gives that 
\begin{align*}
\E[||\by_{\bt}-Y_{\tau}||^{2}\mid\bt\le T] & \le\int_{0}^{T}c(\Delta+||\tau-\bt||^{2})\PP(|\bt-\tau|=s\mid\bt\le T)ds\\
 & \le c\left(\Delta+\int_{0}^{T}s^{2}\PP(|\bt-\tau|=s\mid\bt\le T)ds\right).
\end{align*}
Now we proceed to bound 
\begin{align*}
 \PP(|\bt-\tau|=s\mid\bt\le T) 
=  \frac{\PP(|\bt-\tau|=s,\bt\le T)}{\PP(\bt\le T)}\le\frac{\PP(|\bt-\tau|=s)}{\PP(\bt\le T)}.
\end{align*}
Note that 
\begin{align*}
\PP(\bt>T) & =\int_{0}^{T+1}\PP(\bt>T,\tau=s)ds\\
 & =\int_{0}^{T-1}\PP(\bt>T,\tau=s)ds+\int_{T-1}^{T+1}\PP(\bt>T,\tau=s)ds\\
 & \le(T-1)\PP(|\bt-\tau|\ge1)+\int_{T-1}^{\infty}\PP(\tau=s)ds\\
 & \le(T-1)\E(|\bt-\tau|)+(1-F_{\tau}(T-1))\\
 & \le c\Delta+(1-F_{\tau}(T-1)),
\end{align*}
where $F_{\tau}$ denotes the CDF of $\tau$. When $T$ is properly
large and $\Delta$ is small enough, we have $\PP(\bt>T)\le1/2$ and
thus $\PP(\bt\le T)=1-\PP(\bt>T)\ge1/2$. This implies that 
\[
\PP(|\bt-\tau|=s\mid\bt\le T)\le2\PP(|\bt-\tau|=s).
\]
We thus conclude that 
\begin{align*}
 & \int_{0}^{T}s^{2}\PP(|\bt-\tau|=s\mid\bt\le T)ds\\
\le & 2\int_{0}^{T}s^{2}\PP(|\bt-\tau|=s)ds\\
\le & 2T\int_{0}^{T}s\PP(|\bt-\tau|=s)ds\\
= & 2T\E(|\bt-\tau|)\\
\le & c\Delta.
\end{align*}
We finally conclude that 
\[
\W^{2}[\bar{\pi}_{T},\pi_{T}]\le\E[||\bar{Z}_{\bt}-Z_{\tau}||^{2}\mid\bt\le T]\le\E[||\by_{\bt}-Y_{\tau}||^{2}\mid\bt\le T]\le c\Delta.
\]

\end{proof}

\clearpage

\subsection{Proofs} \label{apx:proof}

\printProofs

\subsection{More Discussions on First Hitting Diffusion Models on $\RR^d$} 
Assume the distribution $\tg$ of interest is on $\RR^{d}$. 
To design first hitting diffusion models that yield results on $\tg$, 
we embed $\RR^d$ into the hyperplane $\Omega \defeq \{(x, y) \in \RR^{d+1} \colon y = \ymax\}$ in $\RR^{d+1}$ where $y_{\max}$ is a constant (e.g., $\ymax=1$). 
We construct a 
baseline process $\bar \Q$ to be a diffusion process on $\RR^{d+1}$: 
\bbb \label{equ:ZtYt}
\bar \Q \colon ~~~~ 
 \df {Z}_t = \df W_t, ~~~~~ 
\df Y_t =  b(Y_t, t) \dt +  \sigma \df \tilde W_t, ~~~~ Z_0 = z_0\in \RR^d, ~~ Y_0 = 0, 
\eee   
where $W_t$ and $\tilde W_t$ are independent Brownian motions in $\RR^{d}$ and $\RR$, respectively. 

We can think $Y_t$ as an ``effective age'' of the particle $\X_t$, 
and the sample is collected when $Y_t = \ymax$. Therefore, the hitting time of interest is $\tau \defeq \{t\colon (X_t,Y_t)\in \Omega \} = \{t\colon Y_t = \ymax \}$. 

A special case is $\sigma =0$ and $b(Y_t, t) = 1$, in which case $(\X_t, Y_t)$ hits the target domain $\Omega$ in the fixed time $t = \ymax$. This corresponds to the standard denoising diffusion models \citep[e.g.,][]{song2020score}. 

Another extreme case  is to take $b=0$, so that $Y_t$ is a Brownian motion without a drift. In this case, 
the hitting time follows an inverse Gamma distribution, 
and the exit distribution is a Cauchy distribution: 
\bb 
\left (\tau  ~|~ \X_t, Y_t\right )\sim  \invgamma\left (\frac{1}{2}, \frac{(\ymax-Y_t )^2}{2\sigma^2} \right), 
&&
\left (\X_{\tau }|~ \X_t, Y_t\right ) \sim \cauchy\left (\X_t, \frac{\ymax-Y_t}{2} \right ), 
\ee 
where the density of $\invgamma(\alpha, \beta)$ is $f(x; \alpha, \beta) =  \frac{\beta^\alpha}{\Gamma(\alpha)} x^{-(\alpha+1)} \exp(-\beta/x)$, and density of 
$\cauchy(\mu, s)$ is 
$f(x; \mu, s) \propto (s^2+\norm{x-\mu}^2)^{-(d+1)/2}$. 

An advantage of using random hitting is that it allows us to spend less time on generating $\X_t$ that is close to the starting point (i.e., small  $\norm{\X_t - x_0}$), and more time on the further points. It allows us to adapt the time based on the ``hardness'' of the target distribution. 

\paragraph{Accelerating the First Hitting Time} 
The inverse Gamma distribution above has a heavy tail 
and occasional causes large hitting time. 
One way to ensure a bounded hitting time is to derive the conditioned process of Brownian motion given that the hitting time $\tau$ is no larger than a threshold. Specifically,  assume $\dist B: \d Y_t = \d W_t$ starting from $Y_0 = y_0 < \ymax$ and $\tau = \inf\{t\colon Y_t = \ymax\}.$ Using $h$-transform, we can show that  $\dist B(\cdot ~|~ \tau \leq T)$ is governed by the following diffusion process: 
 \bb 
 \dist B^T\defeq \dist B(\cdot ~|~ \tau \leq T): ~~~~~~
 & dY_t =\dd_{y} \log \left (1-F\left ( 
 \frac{|\ymax - Y_t|}{\sigma \sqrt{T-t}} \right ) \right ) dt + \sigma \df \tilde W_t, 
 \ee 
 where $F$ is the CDF of standard Gaussian distribution. 

Taking $b(Y_t, t) = \dd_{y} \log \left (1-F\left ( 
 \frac{|\ymax - Y_t|}{\sigma \sqrt{T-t}} \right ) \right )$ in Eq.~\ref{equ:ZtYt}, 
we can obtain the following Poisson kernel for $\bar \Q$: 
$$
\bar \Q (\X_\tau = \d x' ~|~ \X_t=x, Y_t=y)
= \Gamma\left (\alpha, \frac{\phi(x'; x, y)}{T-t} \right ) \frac{\abs{\ymax-Y_t}}{\phi(x'; x, y)^\alpha} \d x',  
$$
where $\alpha = \frac{d+1}{2}$ and $\phi(x'; x, y) = \frac{1}{2}((\ymax-y)^2 + \norm{x'-x}^2)$, 
and $\Gamma(
\alpha, x)$ is the upper incomplete gamma function. 
Correspondingly, the hitting time of this new process is $\invgamma\left (\frac{1}{2}, \frac{(\ymax-Y_t )^2}{2\sigma^2} \right)$ truncated on $[0, T]$.

\end{document}